\documentclass{article}
\usepackage{iclr2027_conference,times}
\newif\ifpapertmlr
\papertmlrfalse
\ifpapertmlr
  \let\PaperSansDefault\sfdefault
  \let\PaperMonoDefault\ttdefault
\fi
\usepackage[scaled=0.92]{helvet}

\usepackage[T1]{fontenc}
\usepackage[scaled=0.9]{fragmentmono}
\ifpapertmlr
  \let\sfdefault\PaperSansDefault
  \let\ttdefault\PaperMonoDefault
\fi

\usepackage{amsmath,amsfonts,bm}

\def\eqref#1{equation~\ref{#1}}

\def\1{\bm{1}}

\DeclareMathAlphabet{\mathsfit}{\encodingdefault}{\sfdefault}{m}{sl}
\SetMathAlphabet{\mathsfit}{bold}{\encodingdefault}{\sfdefault}{bx}{n}

\usepackage{amssymb}
\usepackage{pifont}
\newcommand{\correspondingmark}{%
  \text{\raisebox{0.1ex}{\scalebox{1.4}{\ding{41}}}\kern0.08em}}
\usepackage{algorithm}
\usepackage{algpseudocode}
\usepackage{amsthm}
\newtheorem{theorem}{Theorem}
\newtheorem{lemma}{Lemma}[section]

\theoremstyle{definition}

\theoremstyle{plain}

\usepackage{xcolor}
\definecolor{paperLinkBlue}{HTML}{1F4E79}
\usepackage{hyperref}
\hypersetup{colorlinks=true,allcolors=paperLinkBlue,pdfborder={0 0 0}}
\usepackage{url}
\usepackage{graphicx}
\usepackage{booktabs}
\usepackage{array} %
\usepackage{wrapfig}
\usepackage{capt-of}
\usepackage{enumitem}
\setlist{nosep}
\usepackage[caption=false]{subfig}
\usepackage{varwidth}
\newcommand{\figurepanelfont}{\normalfont\normalsize}
\newcommand{\figurepanelcaptionbox}[2][\linewidth]{%
  \begin{varwidth}{#1}%
  \figurepanelfont\raggedright\noindent#2\par
  \end{varwidth}}
\newcommand{\figurepanelcaption}[1]{%
  \par\begingroup\centering\noindent\figurepanelcaptionbox{#1}\par\endgroup}
\newcommand{\figurepanellabel}[1]{\smash{\subfloat{\label{#1}}}(\thesubfigure)}
\newcommand{\maplabelfont}{\normalfont\fontfamily{phv}\selectfont\small}
\newcommand{\maplabel}[1]{%
  \par\begingroup\maplabelfont\centering\noindent#1\par\endgroup}
\usepackage{threeparttable} %
\AddToHook{env/table/begin}{\setlength{\belowcaptionskip}{8pt}}
\AddToHook{env/table*/begin}{\setlength{\belowcaptionskip}{8pt}}
\usepackage{multicol} %
\usepackage{tikz}
\usetikzlibrary{arrows.meta,bending,positioning,fit,backgrounds,calc,shapes.geometric}
\usepackage{pgfplots}
\pgfplotsset{compat=1.18}
\usepgfplotslibrary{statistics}
\usepackage{colortbl} %
\ifpapertmlr
  \AddToHook{env/tikzpicture/begin}{%
    \renewcommand{\sfdefault}{phv}%
    \renewcommand{\ttdefault}{FragmentMonoSSZeroTwo-TLF}%
    \DeclareFontSeriesDefault[tt]{bf}{b}}
\fi
\tikzset{every picture/.append style={font=\sffamily}}
\tikzset{figure panel caption/.style={font=\figurepanelfont,align=left}}
\tikzset{map label/.style={font=\maplabelfont,align=center}}
\pgfplotsset{
  every axis/.append style={font=\sffamily},
  every axis legend/.append style={font=\sffamily},
}
\definecolor{lorrBlue}{HTML}{0072B2}
\definecolor{lorrOrange}{HTML}{E69F00}
\definecolor{lorrGreen}{HTML}{009E73}
\definecolor{lorrPurple}{HTML}{AA4499}
\definecolor{lorrRed}{HTML}{D55E00}
\definecolor{macRed}{HTML}{FF5F57}    %
\definecolor{macYellow}{HTML}{FEBC2E}
\definecolor{macGreen}{HTML}{28C840}
\newif\ifshowopus
\showopusfalse
\definecolor{claudeOrange}{HTML}{D97757}

\pgfplotsset{lorr development axis/.style={
  axis x line*=bottom,axis y line*=left,
  axis line style={draw=black!55},tick style={draw=black!55},
  label style={font=\scriptsize\sffamily},
  tick label style={font=\scriptsize\sffamily},
  xticklabel={\textsf{\pgfmathprintnumber[assume math mode=true,1000 sep={,},min exponent for 1000 sep=3]{\tick}}},
  yticklabel={\textsf{\pgfmathprintnumber[assume math mode=true,1000 sep={,},min exponent for 1000 sep=3]{\tick}}},
  scaled ticks=false,
}}
\makeatletter
\newcommand{\importasset}[1]{\@ifundefined{importasset@#1}{%
  \global\@namedef{importasset@#1}{}\input{#1}}{}}
\makeatother
\newcommand{\LorrRFiveNMS}{3150.5}
\newcommand{\LorrRFiveKK}{2065.25}
\newcommand{\LorrRFiveMaximum}{4302}
\newcommand{\LorrRFiveMean}{4253.25}
\newcommand{\LorrRFiveMeanLabel}{4,253.3}
\newcommand{\LorrRFiveElapsed}{48.177095467499996}
\newcommand{\LorrRFiveXAxisMax}{45}
\newcommand{\LorrRFiveCommitCount}{319}
\newcommand{\LorrRFiveCommitAxisMax}{350}
\newcommand{\LorrRFiveCommitElapsed}{48.42955136497815}
\newcommand{\LorrRFiveMilestoneCount}{14}
\newcommand{\LorrRFiveMilestoneAxisMax}{15}
\newcommand{\LorrRFiveMilestoneElapsed}{42.338718031644824}
\newcommand{\LorrRFiveNMSLabel}{3,150.5}
\newcommand{\LorrRFiveKKLabel}{2,065.3}
\newcommand{\LorrRFiveMaximumLabel}{4,302}
\newcommand{\LorrRFiveMeanGainLabel}{35.0}
\newcommand{\LorrROneMeanGainLabel}{7.2}
\newcommand{\LorrRFourMeanGainLabel}{6.9}
\newcommand{\LorrRFiveSDLabel}{59.8}

\tikzset{lorr value note/.style={
  font=\scriptsize\sffamily\bfseries,
  align=center,text=black,fill=white,inner sep=1pt,
  anchor=south,yshift=4pt,
}}

\title{Policy Search in Code Space for Multi-Agent Pathfinding}

\author{Matthias Eder$^{1}$, Pantelis Emmanouil$^{2}$, Vivek Farias$^{2}$, Fabian Fischer$^{1}$, Adam Jozefiak$^{2}$, \\
\textbf{Tianyi Peng$^{3}$, Thomas Ulz$^{1}$, Raphael Watschinger$^{1}$, Fangyu Wu$^{3,}$\thanks{Corresponding author: Fangyu Wu (\texttt{fywu85@gmail.com}).}\phantom{$^{\correspondingmark}$}, Andrew Zheng$^{4}$} \\
$^{1}$KNAPP \quad $^{2}$MIT \quad $^{3}$Columbia University \quad $^{4}$University of British Columbia
}

\iclrfinalcopy
\begin{document}
\begingroup
\let\titlefnsymbol\fnsymbol
\renewcommand{\fnsymbol}[1]{%
  \ifnum\value{#1}=1\relax\correspondingmark\else\titlefnsymbol{#1}\fi}
\maketitle
\endgroup

\begin{abstract}
Developing coordination policies for complex multi-agent systems relies on either labor-intensive algorithm design or opaque task-specific reinforcement learning.
With recent advances in autonomous coding agents, we show that policy search directly in code space offers a compelling third path that combines automated optimization with explicit interpretability.
We investigate a human-directed workflow in which coding agents iteratively propose, implement, evaluate, diagnose, and eventually refine algorithms guided by simulation feedback and explicit correctness verification.
We conducted two case studies and found that this workflow robustly outperformed any manually engineered solutions.
On a challenging map in the 2024 League of Robot Runners competition, a two-day search discovered a policy that exceeds the stronger locally reproduced competition entry by 35.0\% on the densest instance and generalizes to unseen map instances.
In simulated real-world warehouses with irregular lanes, continuous robot geometry, and asynchronous motion, a week-long search discovered an entirely new class of multi-agent pathfinding algorithm that outperforms a calibrated proprietary solution in throughput by 58.9\%.
Under stated assumptions, we further prove that the new algorithm is collision-free and completes, in finite time, every accepted goal that is not withdrawn or replaced.
Together, the studies show that policy search in code space powered by state-of-the-art coding agents offers a highly competitive path to building real engineering solutions that are both effective and analytically tractable.

\end{abstract}

\newcommand{\LorrRFiveNMS}{3150.5}
\newcommand{\LorrRFiveKK}{2065.25}
\newcommand{\LorrRFiveMaximum}{4302}
\newcommand{\LorrRFiveMean}{4253.25}
\newcommand{\LorrRFiveMeanLabel}{4,253.3}
\newcommand{\LorrRFiveElapsed}{48.177095467499996}
\newcommand{\LorrRFiveXAxisMax}{45}
\newcommand{\LorrRFiveCommitCount}{319}
\newcommand{\LorrRFiveCommitAxisMax}{350}
\newcommand{\LorrRFiveCommitElapsed}{48.42955136497815}
\newcommand{\LorrRFiveMilestoneCount}{14}
\newcommand{\LorrRFiveMilestoneAxisMax}{15}
\newcommand{\LorrRFiveMilestoneElapsed}{42.338718031644824}
\newcommand{\LorrRFiveNMSLabel}{3,150.5}
\newcommand{\LorrRFiveKKLabel}{2,065.3}
\newcommand{\LorrRFiveMaximumLabel}{4,302}
\newcommand{\LorrRFiveMeanGainLabel}{35.0}
\newcommand{\LorrROneMeanGainLabel}{7.2}
\newcommand{\LorrRFourMeanGainLabel}{6.9}
\newcommand{\LorrRFiveSDLabel}{59.8}

\newcommand{\CgarCommitCount}{50}
\newcommand{\CgarCommitElapsed}{441.091111111}
\newcommand{\CgarCommitAxisMax}{55}
\newcommand{\CgarTimeAxisMax}{450}
\newcommand{\CgarActiveElapsed}{179.881418708}
\newcommand{\CgarActiveAxisMax}{190}
\newcommand{\CgarMilestoneCount}{15}
\newcommand{\CgarMilestoneAxisMax}{16}

\newcommand{\CgarScalingRuns}{800}
\newcommand{\CgarScalingMaxFleet}{100}
\newcommand{\CgarScalingBaseline}{441}
\newcommand{\CgarScalingBaselineFleet}{65}
\newcommand{\CgarScalingMatchBaselineFleet}{26}
\newcommand{\CgarScalingMatchBaselineMean}{441.6}
\newcommand{\CgarScalingMatchBaselineShare}{40}
\newcommand{\CgarScalingMatchKnappFleet}{30}
\newcommand{\CgarScalingMatchKnappMean}{492.8}
\newcommand{\CgarScalingKnappMean}{487.9}
\newcommand{\CgarScalingPeakFleet}{68}
\newcommand{\CgarScalingPeakMean}{784.2}
\newcommand{\CgarScalingLastMean}{415.8}

\newsavebox{\GaGridFull}
\newsavebox{\GaGridMap}
\newsavebox{\GaWarehouseMap}
\sbox{\GaGridFull}{\input{assets/lorr/random-03-map.tex}}
\sbox{\GaGridMap}{\begin{tikzpicture}[x=0.195cm,y=0.195cm]
  \useasboundingbox (21,10) rectangle (31,20);
  \begin{scope}
    \clip (21.02,10.02) rectangle (30.98,19.98);
    \node[anchor=south west,inner sep=0pt,outer sep=0pt,
      xshift=-0.2746pt,yshift=-0.2746pt] at (0,0)
      {\scalebox{2.7465}{\usebox{\GaGridFull}}};
    \draw[black!18,line width=0.2pt,step=1] (21,10) grid (31,20);
  \end{scope}
  \draw[black!45,line width=0.4pt] (21,10) rectangle (31,20);
\end{tikzpicture}}
\sbox{\GaWarehouseMap}{\includegraphics[height=2.1cm]{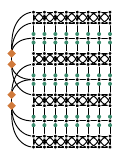}}
\newsavebox{\GaUserIcon}
\newsavebox{\GaCodeIcon}
\sbox{\GaUserIcon}{\begin{tikzpicture}
  \fill[lorrBlue] (0,0.36) circle (0.12);
  \fill[lorrBlue] (-0.22,0) arc[start angle=180,end angle=0,radius=0.22]
    -- cycle;
\end{tikzpicture}}
\sbox{\GaCodeIcon}{\begin{tikzpicture}
  \fill[white,draw=black!45,line width=0.4pt]
    (0,0) -- (0.38,0) -- (0.38,0.38) -- (0.28,0.48) -- (0,0.48) -- cycle;
  \draw[black!45,line width=0.4pt] (0.28,0.48) -- (0.28,0.38) -- (0.38,0.38);
  \foreach \y/\x/\l/\c in {0.37/0.06/0.14/lorrBlue,0.27/0.10/0.18/lorrBlue!55,
      0.17/0.10/0.12/black!35,0.07/0.06/0.22/lorrBlue}
    \draw[\c,line width=0.9pt] (\x,\y) -- ++(\l,0);
\end{tikzpicture}}

\newcommand{\GaCodeFont}{\fontsize{5}{6}\selectfont\ttfamily
  \spaceskip=\fontcharwd\font`\x\relax\xspaceskip=\spaceskip}

\pgfplotsset{ga axis/.style={
  lorr development axis,
  scale only axis,anchor=south west,
  width=3.1cm,height=1.45cm,clip=false,
  xlabel style={yshift=2pt},ylabel style={yshift=-3pt},
}}

\tikzset{
  ga row/.style={fill=black!4,draw=none,rounded corners=3pt},
  ga title/.style={font=\footnotesize\sffamily\bfseries,anchor=north west,
    inner sep=0pt},
  ga result/.style={font=\scriptsize\sffamily,text=lorrBlue!80!black,
    anchor=north west,inner sep=0pt},
  ga step/.style={font=\scriptsize\sffamily,inner sep=0.5pt,
    text height=1.6ex,text depth=0.45ex},
  ga ring/.style={lorrBlue,line width=1pt,
    -{Triangle[length=3.3pt,width=3.3pt,bend]}},
  ga code/.style={font=\GaCodeFont,anchor=west,inner sep=0pt},
  ga codeline/.style={ga code,text=white!82!black},
  ga arrow/.style={lorrBlue,line width=0.8pt,
    -{Computer Modern Rightarrow},shorten <=2pt,shorten >=2pt},
  ga block/.style={font=\scriptsize\sffamily,inner sep=0pt},
  ga number/.style={font=\Large\sffamily,inner sep=0pt},
  ga subcaption/.style={figure panel caption,anchor=north,
    inner sep=0pt},
  ga baseline/.style={font=\scriptsize\sffamily,text=black!80,fill=white,
    inner xsep=1pt,inner ysep=0pt,anchor=north east,yshift=-#1},
  ga baseline/.default=2pt,
  ga value/.style={font=\scriptsize\sffamily\bfseries,text=black,fill=white,
    inner sep=1pt},
}

\begin{center}
\makeatletter\def\@captype{figure}\makeatother
\begin{tikzpicture}

\node[ga subcaption] at (3.44,0.05) {\figurepanelcaptionbox[6.86cm]{%
\figurepanellabel{fig:graphical-abstract:lorr} League of Robot Runners 2024.
We exceed the top entry by roughly 35\% in a dense map.}};
\node at (0.95,1.4) {\scalebox{0.82}{\usebox{\GaGridMap}}};
\begin{axis}[ga axis,
  at={(3.15cm,0.85cm)},
  xlabel={Milestone Commit},ylabel={Tasks\\Completed},ylabel style={align=center},
  xmin=0,xmax=\LorrRFiveMilestoneAxisMax,ymin=0,ymax=4800,
  xtick={0,5,10,15},ytick={0,\LorrRFiveKK,\LorrRFiveNMS,4000},
  yticklabel={\textsf{\pgfmathprintnumber[assume math mode=true,fixed,precision=0,1000 sep={,},min exponent for 1000 sep=3]{\tick}}},
]
  \draw[black!55,line width=0.6pt]
    (axis cs:0,\LorrRFiveNMS) -- (axis cs:\LorrRFiveMilestoneAxisMax,\LorrRFiveNMS);
  \draw[black!55,line width=0.6pt]
    (axis cs:0,\LorrRFiveKK) -- (axis cs:\LorrRFiveMilestoneAxisMax,\LorrRFiveKK);
  \node[ga baseline=0.9pt,xshift=2pt] at (axis cs:\LorrRFiveMilestoneAxisMax,\LorrRFiveNMS)
    {Best Entry};
  \node[ga baseline=0.9pt,xshift=2pt] at (axis cs:\LorrRFiveMilestoneAxisMax,\LorrRFiveKK)
    {Second-Best Entry};
  \addplot[color=lorrBlue,const plot,solid,line width=1.1pt,no marks]
    table[x=milestone,y=tasks,col sep=comma]{assets/lorr/random-05-milestones-reported.csv};
  \addplot[only marks,color=lorrBlue,mark=*,mark size=1.2pt,
    mark options={solid,fill=lorrBlue}]
    table[x=milestone,y=tasks,col sep=comma]{assets/lorr/random-05-milestones-reported.csv};
  \node[ga value,anchor=south west,xshift=1pt] at (axis cs:0,54) {54};
  \node[ga value,anchor=west,xshift=2pt]
    at (axis cs:14,\LorrRFiveMean) {\LorrRFiveMeanLabel};
\end{axis}

\node[ga subcaption] at (10.53,0.05) {\figurepanelcaptionbox[6.86cm]{%
\figurepanellabel{fig:graphical-abstract:warehouse} Industrial warehouse. We complete 58.9\%\ifpapertmlr\space\else\\\fi
more orders than the proprietary solution.}};
\node at (8.035,1.3) {\usebox{\GaWarehouseMap}};
\begin{axis}[ga axis,
  at={(9.935cm,0.85cm)},
  xlabel={Milestone Commit},ylabel={Orders/Hour},
  xmin=0,xmax=\CgarMilestoneAxisMax,ymin=0,ymax=850,
  xtick={0,5,10,15},ytick={0,\CgarScalingKnappMean,800},
]
  \draw[black!55,line width=0.6pt]
    (axis cs:0,\CgarScalingKnappMean) -- (axis cs:\CgarMilestoneAxisMax,\CgarScalingKnappMean);
  \node[font=\scriptsize\sffamily,text=black!80,fill=white,inner xsep=1pt,
    inner ysep=0pt,anchor=north east,yshift=-2pt,align=right]
    at (axis cs:\CgarMilestoneAxisMax,\CgarScalingKnappMean) {Proprietary\\Solution};
  \addplot[color=lorrBlue,const plot,solid,line width=1.1pt,no marks]
    table[col sep=comma,x=milestone,y=throughput]
      {assets/cgar/cgar-checkpoint-commits.csv};
  \addplot[only marks,color=lorrBlue,mark=*,mark size=1.2pt,
    mark options={solid,fill=lorrBlue}]
    table[col sep=comma,x=milestone,y=throughput]
      {assets/cgar/cgar-checkpoint-commits.csv};
  \node[ga value,anchor=south west,xshift=1pt] at (axis cs:0,0) {0};
  \node[ga value,anchor=west,xshift=2pt] at (axis cs:15,775.3125) {775.3};
\end{axis}
\end{tikzpicture}

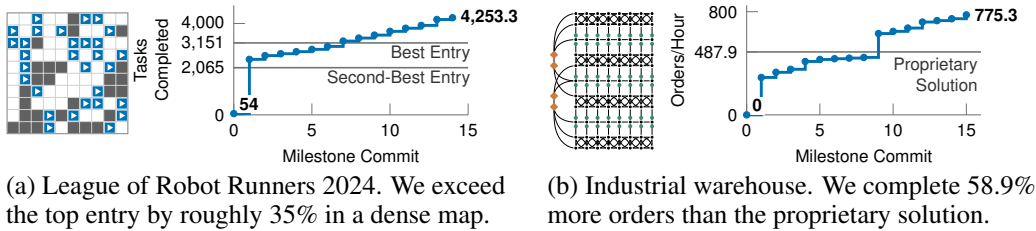
\captionof{figure}{Throughput of the discovered policies across milestone commits.
Human-directed coding agents iteratively refine coordination algorithms, producing policies that outperform strong manually engineered baselines in competitive multi-agent pathfinding problems.
Maps above are illustrative.}
\label{fig:graphical-abstract}
\end{center}

\section{Introduction}
\label{sec:intro}

Multi-agent coordination underlies warehouse fulfillment, factory automation, and parcel sorting.
In lifelong multi-agent pathfinding (LMAPF), tasks arrive continuously, requiring a fleet of robots to coordinate both task allocation and collision-free navigation through a shared spatial network~\citep{wurman2008coordinating,ma2017lifelong,li2021rhcr}.
The operational objective is to maximize system throughput subject to collision avoidance and strict real-time compute budgets.
Because uncoordinated routing decisions cause severe congestion and deadlocks, effective warehouse operation typically requires sophisticated joint optimization of task assignment and trajectory coordination across the fleet.

Existing methods generally fall into two paradigms.
First, the standard approach relies on human experts to engineer sophisticated solvers, dispatch rules, and traffic guidance heuristics~\citep{okumura2022priority,jiang2024scaling,zhang2024guidance}.
Although effective, manual engineering is bottlenecked by human engineering throughput and cannot scale algorithm discovery with available compute.
Second, learning-based methods attempt to bypass this manual effort by training neural policies directly through simulated interaction~\citep{damani2021primal2,skrynnik2024follower,andreychuk2025mapfgpt}.
However, neural policies often lack formal safety guarantees, remain difficult to inspect and verify, and resist modular modification when operational constraints shift.

We study a third approach: \emph{policy search in code space} (PSCS), driven by coding agents.
In this paradigm, human designers specify the system objective, physical constraints, and scaffolding interfaces, while coding agents iteratively refine the implementation of the planner and scheduler directly in source code, which we refer to as \emph{code space}.
Guided by simulation benchmarks, unit tests, and execution traces, the agent discovers procedural algorithms that are computationally efficient, interpretable, and maintainable.
Unlike manual design, the search scales directly with compute; unlike black-box learning, the resulting policy remains fully transparent and verifiable.

Automated code generation has demonstrated promise in discovering novel mathematical algorithms and discrete heuristics~\citep{romeraparedes2024funsearch,novikov2025alphaevolve,liu2024eoh}.
However, existing studies focus primarily on idealized theoretical benchmarks or stylized combinatorial problems.
It remains an open question whether PSCS can scale to demanding cyber-physical systems such as LMAPF, where task allocation, spatial planning, and low-latency runtime engineering must interact robustly under realistic constraints.
Specifically, we investigate whether coding agents can navigate these coupled requirements across both standardized benchmarks and complex industrial environments.

We evaluate PSCS across two distinct domains.
The first is the 2024 League of Robot Runners (LoRR) benchmark~\citep{chan2024lorr}, a synchronous grid competition environment operating with $100$ to $800$ robots.
In this domain, PSCS systematically refines existing baselines to discover layout-specific routing heuristics, outperforming the stronger locally reproduced competition entry by up to 35.0\% in mean throughput (Figure~\ref{fig:graphical-abstract:lorr}).
The second is a high-fidelity simulation of an operational warehouse from KNAPP characterized by irregular lanes, continuous vehicle geometry, and asynchronous execution.
In a challenging production map, our synthesized algorithm outperforms an optimized proprietary baseline by 58.9\% (Figure~\ref{fig:graphical-abstract:warehouse}) while preserving provable progress certificates against deadlock under stated assumptions.

The contributions of this work are as follows:
\textsl{1}) We demonstrate the feasibility and effectiveness of policy search in code space (PSCS) for automated, nontrivial combinatorial algorithm design in LMAPF.
\textsl{2}) On the LoRR 2024 benchmark, we show that PSCS effectively refines established planners to discover map-specific heuristics, outperforming the stronger locally reproduced competition entry by up to 35.0\% in mean throughput.
\textsl{3}) In high-fidelity industrial warehouse environments, we show that PSCS discovers an entirely new class of algorithms that outperforms an optimized proprietary baseline by 58.9\% while maintaining provable progress certification under stated assumptions.

\section{Related Work}
\label{sec:related_work}

This work sits at the intersection of multi-agent pathfinding (MAPF) and automated program synthesis.
Rather than hand-crafting domain heuristics or optimizing parametric weights, we search for coordination policies directly in code space.

\paragraph{Multi-Agent Pathfinding.}
MAPF seeks collision-free trajectories for a team of agents~\citep{stern2019mapf}.
In lifelong MAPF (LMAPF), agents continuously receive new tasks upon completing previous assignments~\citep{ma2017lifelong,li2021rhcr}.
Because optimal solvers such as conflict-based search~\citep{sharon2015cbs} do not scale to hundreds of agents, large-scale deployments rely on fast suboptimal planners that human experts design and tune for specific topologies.
Representative rule-based methods include PIBT~\citep{okumura2022priority} for one-step priority inheritance, Causal-PIBT~\citep{okumura2021causalpibt} for asynchronous execution, and WPPL~\citep{jiang2024scaling}, which combines windowed PIBT with large neighborhood search~\citep{li2021lns,li2022lns2} and traffic guidance graphs.
Although guidance graph optimization~\citep{zhang2024guidance} automates edge-cost tuning over these graphs, the underlying planning algorithms remain fixed.
Alternatively, learning-based formulations train neural policies via reinforcement or imitation learning~\citep{sartoretti2019primal,damani2021primal2,skrynnik2024follower,andreychuk2025mapfgpt}, with recent neural planners and schedulers outperforming earlier benchmark baselines~\citep{jiang2025sillm,gaber2026grand}.
Large language models (LLMs) have also been explored as direct online planners; however, prompting LLMs to emit next-step moves sequentially struggles to scale beyond small stylized instances~\citep{chen2025llmmapf}.

\paragraph{Automated Program Synthesis.}
Policy search in code space builds upon a broader lineage of program synthesis and algorithmic discovery.
Early programmatic reinforcement learning formulated policies as interpretable programs, though restricted to small, domain-specific languages~\citep{verma2018pirl,trivedi2021leaps}.
Modern LLMs lift these grammar constraints, enabling coding agents to manipulate general-purpose programming languages such as C++ and Python directly.
Recent frameworks leverage this capability to evolve individual functions or isolated files for discrete algorithms and combinatorial optimization heuristics~\citep{romeraparedes2024funsearch,novikov2025alphaevolve,liu2024eoh,ye2024reevo}.
Complementary lines of work optimize prompt graphs and textual variables through verbal feedback and execution traces~\citep{yuksekgonul2025textgrad,agrawal2026gepa}.
More recently, automated loops have advanced to whole programs and repositories, in SAT solving, model training, and optimization contests~\citep{yu2025satlution,karpathy2026autoresearch,imajuku2025alebench}, as well as dispatch and charging policies for automated guided vehicles without spatial path planning~\citep{li2026traces}.

Our study differs from existing literature in the following aspects.
\textsl{1}) Whereas existing automated MAPF methods predominantly tune heuristic parameters or train parametric neural weights, our work searches over planner and scheduler \emph{implementations} directly in code space.
\textsl{2}) Furthermore, while prior program synthesis frameworks rely on sophisticated evolutionary harnesses for stylized benchmarks, we demonstrate that state-of-the-art coding agents discover effective algorithms for demanding LMAPF problems without rigid search harnesses.
\section{Problem Setup}
\label{sec:problem_setup}

We consider an environment consisting of $n$ agents, a finite directed configuration graph $G=(V,E)$, and a discrete planning horizon of $T$ clock ticks.
Each vertex $u \in V$ represents an agent configuration, such as a discrete position and orientation.
Each directed edge $e=(u,v) \in E$ denotes a feasible motion with strictly positive traversal duration $\delta(e)\in\mathbb{Z}_{>0}$ ticks.
At any tick $t \in \{0, \dots, T\}$, an agent is either stationary at a vertex $u \in V$ or traversing an edge $e \in E$ with $\tau \in \{1, \dots, \delta(e)\}$ ticks remaining before arrival.
Waiting is permitted only at vertices; edge traversals are non-preemptive and must proceed until completion.

Collision avoidance is formalized through a spatial conflict map.
Let $\Gamma: V\cup E \to 2^{V\cup E}$ denote a conflict function mapping any occupied vertex or edge to the set of configurations and edges that no other agent may occupy simultaneously.
A valid multi-agent configuration requires that for any pair of agents occupying elements $x, y \in V \cup E$, $y \notin \Gamma(x)$.
A stationary agent at vertex $u$ occupies only $u$.
A moving agent traversing $e=(u,v)$ occupies edge $e$ throughout traversal, its departure vertex $u$ at the initiation tick, and its arrival vertex $v$ at the completion tick.
This general mapping captures classical grid-based vertex and swap conflicts, as well as complex footprint conflicts in continuous warehouse layouts.

Tasks enter an available-task pool $Q$ over time according to an exogenous stochastic arrival process.
Each task consists of an ordered sequence of \emph{errands}.
An errand is a tuple $(W, \sigma)$, where $W \subseteq V$ denotes the set of admissible goal configurations and $\sigma \in \mathbb{Z}_{\ge 0}$ is the required service duration.
An agent completes an errand by dwelling at a chosen vertex $w \in W$ for $\sigma$ contiguous ticks before proceeding to the subsequent errand.
Task assignment is constrained such that each agent serves at most one active task at any time, and reassignments are prohibited once execution of the initial errand begins.

A coordination policy $\pi$ assigns available tasks to idle agents and selects, for each agent stationed at a vertex, either a wait action or an outgoing edge transition.
Let $J(\pi)$ denote the cumulative number of tasks completed by horizon $T$.
The primary objective is to synthesize a policy $\pi$ that maximizes expected throughput $\mathbb{E}[J(\pi)]$, subject to collision avoidance, valid task assignments, kinematic motion constraints, and real-time execution budgets.
For domain-specific instantiations and operational constraints, see Sections~\ref{sec:lorr:problem_setup} and~\ref{sec:warehouse:problem_setup}.

\section{Policy Search in Code Space}
\label{sec:method}

In policy search in code space (PSCS), we search for the source code of the coordination policy itself (Figure~\ref{fig:method}).

We initiate the search by specifying the environment, operational constraints, performance objectives, and an initial algorithm or reference implementation.
Rather than using a rigid, specialized evolutionary harness, we prompt the coding agent directly in natural language, with minimal steering or redirecting as the search progresses.
Remarkably, we find state-of-the-art coding agents already effective without additional augmentation from an external search harness.

Once the search is initiated, the coding agent iteratively refines the codebase through a four-stage loop:
\textsl{1}) proposing algorithmic edits,
\textsl{2}) implementing source-level modifications,
\textsl{3}) evaluating candidates in closed-loop simulation, and
\textsl{4}) diagnosing bottlenecks from execution traces and profiler logs.
Importantly, language models are queried strictly offline during development; policy execution at runtime is handled entirely by compiled binaries with zero model inference latency.

\begin{figure}[t]
\centering
\begin{tikzpicture}
\node[ga number] at (0.3,4.98) {1};
\node[inner sep=0pt] at (0.85,4.98) {\usebox{\GaUserIcon}};
\node[ga block,anchor=west,font=\scriptsize\sffamily\bfseries] at (1.3,5.5)
  {User};
\node[ga block,anchor=north west,align=left] at (1.3,5.3)
  {1. Env\\2. Constraint\\3. Objective\\4. Seed Algorithm};
\draw[ga arrow] (3.1,4.98) -- (4.03,4.98);

\node[ga number] at (4.46,4.98) {2};
\fill[black!80,rounded corners=1.5pt] (4.71,4.77) rectangle (5.31,5.19);
\node[font=\scriptsize\ttfamily\bfseries,text=white,anchor=west,
  inner sep=0pt] at (4.8,4.98) {>\_};
\node[ga block,anchor=west,font=\scriptsize\sffamily\bfseries] at (5.51,5.5)
  {Coding Agent};
\node[ga step,anchor=west] at (5.51,5.07) {Propose};
\node[ga step,anchor=west] at (7.29,5.07) {Implement};
\node[ga step,anchor=west] at (7.29,4.45) {Evaluate};
\node[ga step,anchor=west] at (5.51,4.45) {Diagnose};
\coordinate (ring) at (6.87,4.76);
\foreach \from in {125,35,-55,-145}
  \draw[ga ring] (ring) ++(\from:0.32)
    arc[start angle=\from,end angle=\from-70,radius=0.32];
\draw[ga arrow] (8.62,4.98) -- (9.55,4.98);

\node[ga number] at (10,4.98) {3};
\node[inner sep=0pt] at (10.5,4.98) {\usebox{\GaCodeIcon}};
\node[ga block,anchor=west,font=\scriptsize\sffamily\bfseries] at (10.95,5.5)
  {Policy as Code};
\begin{scope}
  \clip[rounded corners=3pt] (10.95,4.35) rectangle (13.8,5.29);
  \fill[black!82] (10.95,4.35) rectangle (13.8,5.29);
  \fill[black!66] (10.95,5.09) rectangle (13.8,5.29);
\end{scope}
\fill[macRed] (11.08,5.19) circle (0.042);
\fill[macYellow] (11.22,5.19) circle (0.042);
\fill[macGreen] (11.36,5.19) circle (0.042);
\node[ga codeline] at (11.03,4.91) {const size\_t attempts =};
\node[ga codeline] at (11.03,4.72) {\phantom{xx}allow\_detour ? kMax...};
\node[ga codeline] at (11.03,4.53)
  {\phantom{xxxxxxxxxxxxxxx}: size\_t\{1\};};
\end{tikzpicture}
\caption{Policy search in code space.
The user provides constraints, performance objectives, and an initial codebase.
The coding agent iteratively modifies source code, benchmarks candidates in simulation, and inspects traces to guide subsequent revisions.}
\label{fig:method}
\end{figure}
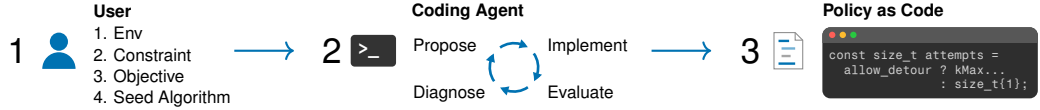

During the search, all candidate policies must pass the following hierarchical verification before acceptance.
First, a candidate must compile cleanly and pass unit tests enforcing collision avoidance, deadlock prevention, and compute budgets.
Next, the coding agent uses fast short-horizon simulations to filter unpromising candidates.
Lastly, surviving candidates undergo longer, multi-seed runs to confirm throughput improvements across stochastic task arrivals.
Any candidate that fails the verification is rejected.

We apply this workflow to two distinct case studies: refining competitive heuristics on an LMAPF competition benchmark (Section~\ref{sec:lorr}), and discovering a routing policy for industrial warehouses with a provable progress guarantee under stated assumptions (Section~\ref{sec:warehouse}).

\newcommand{\LorrFinalRows}{%
RANDOM-01 & 644.8 & 8.1 & 687.1 & 6.8 & 736.8 & 10.4 & 724 & 754 & $+7.22\%$ \\
RANDOM-02 & 1,207.6 & 16.1 & 1,219.1 & 18.0 & 1,356.9 & 14.6 & 1,332 & 1,380 & $+11.30\%$ \\
RANDOM-03 & 2,384.1 & 42.3 & 1,956.6 & 20.1 & 2,647.5 & 31.4 & 2,606 & 2,703 & $+11.05\%$ \\
RANDOM-04 & 2,610.8 & 33.6 & 1,279.9 & 46.4 & 2,790.1 & 54.9 & 2,716 & 2,886 & $+6.87\%$ \\
RANDOM-05 & 3,150.5 & 69.7 & 2,065.3 & 25.5 & 4,253.3 & 59.8 & 4,170 & 4,331 & $+35.00\%$ \\%
}

\newcommand{\LorrRFiveNMS}{3150.5}
\newcommand{\LorrRFiveKK}{2065.25}
\newcommand{\LorrRFiveMaximum}{4302}
\newcommand{\LorrRFiveMean}{4253.25}
\newcommand{\LorrRFiveMeanLabel}{4,253.3}
\newcommand{\LorrRFiveElapsed}{48.177095467499996}
\newcommand{\LorrRFiveXAxisMax}{45}
\newcommand{\LorrRFiveCommitCount}{319}
\newcommand{\LorrRFiveCommitAxisMax}{350}
\newcommand{\LorrRFiveCommitElapsed}{48.42955136497815}
\newcommand{\LorrRFiveMilestoneCount}{14}
\newcommand{\LorrRFiveMilestoneAxisMax}{15}
\newcommand{\LorrRFiveMilestoneElapsed}{42.338718031644824}
\newcommand{\LorrRFiveNMSLabel}{3,150.5}
\newcommand{\LorrRFiveKKLabel}{2,065.3}
\newcommand{\LorrRFiveMaximumLabel}{4,302}
\newcommand{\LorrRFiveMeanGainLabel}{35.0}
\newcommand{\LorrROneMeanGainLabel}{7.2}
\newcommand{\LorrRFourMeanGainLabel}{6.9}
\newcommand{\LorrRFiveSDLabel}{59.8}

\section{Case Study 1: League of Robot Runners}
\label{sec:lorr}

Our first case study investigates whether PSCS can systematically discover and refine high-performance heuristics for established LMAPF benchmarks.
We evaluate the framework on five challenging instances from the 2024 League of Robot Runners (LoRR) competition, wherein a single policy must jointly perform task allocation and path planning~\citep{lorr2024archive}. At the time of writing this paper, the post-2024 LoRR competition archives have not yet been made public.

\subsection{Experiment Setup}
\label{sec:lorr:problem_setup}

The LMAPF setting in LoRR follows the classical instantiation of the formulation in Section~\ref{sec:problem_setup}.
A vertex represents a free cell with one of four cardinal orientations, yielding $|V|=3{,}276$.
Each directed edge represents either a forward translation to an adjacent open cell or a $90^\circ$ in-place rotation, each with duration $\delta(e)=1$ tick.
Two agents conflict when they occupy the same spatial cell or traverse an edge in opposing directions simultaneously; thus, $\Gamma(u)$ contains all four directional vertices at the corresponding grid coordinate.
An errand target $W$ comprises the four directional vertices corresponding to a designated service cell, with service time $\sigma=0$.
At step $t=0$, $1.5n$ tasks are visible, and each fulfilled task immediately reveals a replacement task.
The dispatcher selects assignments from the available-task pool $Q$, without a FIFO or task-level fairness requirement.
The instances span $n \in \{100, 200, 400, 700, 800\}$ robots, evaluated over planning horizons $T \in \{600, 600, 800, 1{,}000, 2{,}000\}$ ticks, respectively.

The policy search was initiated by supplying the coding agent, Codex Astra, with standard planning baselines: an implementation of Priority Inheritance with Backtracking (PIBT)~\citep{okumura2022priority}, the windowed large neighborhood search framework with guidance graphs from the 2023 winning entry WPPL~\citep{jiang2024scaling}, and task-dispatching routines from 2024 participants.
The environments consist of the five instances RANDOM-01--05, which share a common $32\times32$ 2D grid containing 819 traversable cells (Figure~\ref{fig:lorr:maps}).
The agent operated in a command-line environment equipped with compilation toolchains, local simulation scripts, and trajectory verifiers.
The search objective was to maximize cumulative throughput while strictly respecting collision avoidance, valid assignment constraints, a 1.0\,s per-step execution budget, and a 32\,GB memory ceiling.

\begin{figure}[t]
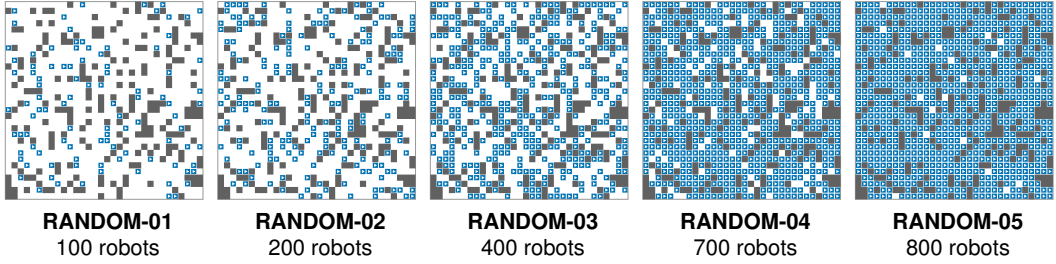

\centering
\begin{minipage}{0.195\linewidth}\centering
\resizebox{\linewidth}{!}{\input{assets/lorr/random-01-map.tex}}\par\vspace{1pt}
\maplabel{\textbf{RANDOM-01}\\100 robots}
\end{minipage}\hfill
\begin{minipage}{0.195\linewidth}\centering
\resizebox{\linewidth}{!}{\input{assets/lorr/random-02-map.tex}}\par\vspace{1pt}
\maplabel{\textbf{RANDOM-02}\\200 robots}
\end{minipage}\hfill
\begin{minipage}{0.195\linewidth}\centering
\resizebox{\linewidth}{!}{\input{assets/lorr/random-03-map.tex}}\par\vspace{1pt}
\maplabel{\textbf{RANDOM-03}\\400 robots}
\end{minipage}\hfill
\begin{minipage}{0.195\linewidth}\centering
\resizebox{\linewidth}{!}{\input{assets/lorr/random-04-map.tex}}\par\vspace{1pt}
\maplabel{\textbf{RANDOM-04}\\700 robots}
\end{minipage}\hfill
\begin{minipage}{0.195\linewidth}\centering
\resizebox{\linewidth}{!}{\input{assets/lorr/random-05-map.tex}}\par\vspace{1pt}
\maplabel{\textbf{RANDOM-05}\\800 robots}
\end{minipage}
\caption{The five RANDOM instances at $t=0$.
Gray cells are obstacles, blue cells hold robots.}
\label{fig:lorr:maps}
\end{figure}

\subsection{Results and Analysis}
\label{sec:lorr:results}

We benchmark our synthesized policy, termed Pipelined Look-Ahead with Task Matching (PLATM), against No Man's Sky (NMS) and Kitty Knight (KK) in the combined track, the top two winning entries of LoRR 2024.
Both baseline entries were re-evaluated on identical maps, initial poses, task sequences, and horizons using identical hardware: 16 physical CPU cores, a 32\,GB memory ceiling, and a strict 1.0\,s execution budget per step.
Evaluations comprise eight paired trials per instance using fixed initial poses alongside identical random seeds for both the planner and the task arrival process.

Table~\ref{tab:lorr:final} summarizes the final throughput comparison across all five competition instances.
PLATM consistently outperforms both expert-engineered baselines across all fleet densities, achieving performance gains of up to 35.0\% over the stronger baseline.
The performance advantage is particularly pronounced on the densest instance ($n = 800$), where mutual blocking and congestion collapse severely degrade standard heuristic solvers.

Figure~\ref{fig:lorr:dev} illustrates the development trajectory of PLATM on RANDOM-05.
Initial candidates failed due to exponential backtracking timeouts and graph deadlocks under standard PIBT at only 54 tasks.
The agent systematically resolved these bottlenecks in stages: having waiting robots turn toward their next move in advance ($\sim$2{,}400 tasks), evaluating more candidate plans per step ($\sim$3{,}000 tasks), and scoring candidate moves across short simulated rollouts ($>$4{,}000 tasks).
To adapt to less crowded maps (RANDOM-01--04), the agent enabled windowed group replanning and eliminated thread-creation overhead to stay strictly within the one-second time limit.
The entire search converged in 42.3 wall-clock hours, consuming 3.5M output tokens ($\approx$\$700).

The synthesized policy couples task allocation and spatial planning in a unified two-level architecture:
At the upper level, the task dispatcher balances assignments using a rotation-aware cost metric rather than simple grid distance, while allowing robots to trade unstarted goals dynamically as spatial conditions change.
At the lower level, the motion planner is matched to traffic density: it coordinates the fleet one step at a time on the crowded instances, but uses windowed group replanning on the sparser ones.
Complementing this adaptive planning, waiting robots use idle ticks to rotate toward their intended departure headings in advance, eliminating rotational delay the moment a path opens.

\begin{table}[t]
\centering
\caption{Tasks completed by PLATM and the top LoRR entries, as mean $\pm$ s.d.\ over eight paired trials.}
\label{tab:lorr:final}
\small
\setlength{\tabcolsep}{4pt}
\begin{tabular}{lr@{\,\(\pm\)\,}lr@{\,\(\pm\)\,}l>{\bfseries}r@{\,\(\pm\)\,}>{\bfseries}lrrr}
\toprule
 & \multicolumn{4}{c}{Top LoRR Entries} & \multicolumn{4}{c}{PLATM} & \\
\cmidrule(lr){2-5}\cmidrule(lr){6-9}
 & \multicolumn{2}{c}{NMS} & \multicolumn{2}{c}{KK} & \multicolumn{1}{r@{\,\(\pm\)\,}}{Mean} & \multicolumn{1}{@{}l}{S.D.} & Min & Max & Gain \\
\midrule
\LorrFinalRows
\bottomrule
\end{tabular}
\end{table}

\begin{figure}[t]
\centering
\begin{minipage}[b]{0.49\linewidth}\centering
\begin{tikzpicture}
\begin{axis}[
  lorr development axis,
  scale only axis,width=0.78\linewidth,height=1.75cm,
  xlabel={Milestone Commit},ylabel={Tasks Completed},
  xmin=0,xmax=\LorrRFiveMilestoneAxisMax,ymin=0,ymax=4800,
  xtick={0,5,10,15},ytick={0,1000,2000,3000,4000},clip=false,
]
\draw[black!55,solid,line width=0.6pt]
  (axis cs:0,\LorrRFiveNMS) -- (axis cs:\LorrRFiveMilestoneAxisMax,\LorrRFiveNMS);
\draw[black!55,solid,line width=0.6pt]
  (axis cs:0,\LorrRFiveKK) -- (axis cs:\LorrRFiveMilestoneAxisMax,\LorrRFiveKK);
\node[font=\scriptsize\sffamily,text=black!80,fill=white,inner xsep=2pt,
  inner ysep=0pt,anchor=north east,yshift=-2pt]
  at (axis cs:\LorrRFiveMilestoneAxisMax,\LorrRFiveNMS) {NMS};
\node[font=\scriptsize\sffamily,text=black!80,fill=white,inner xsep=2pt,
  inner ysep=0pt,anchor=north east,yshift=-2pt]
  at (axis cs:\LorrRFiveMilestoneAxisMax,\LorrRFiveKK) {KK};
\addplot[color=lorrBlue,const plot,solid,line width=1.25pt,no marks]
 table[x=milestone,y=tasks,col sep=comma]{assets/lorr/random-05-milestones-reported.csv};
\addplot[only marks,color=lorrBlue,mark=*,mark size=1.5pt,
 mark options={solid,fill=lorrBlue}]
 table[x=milestone,y=tasks,col sep=comma]{assets/lorr/random-05-milestones-reported.csv};
\input{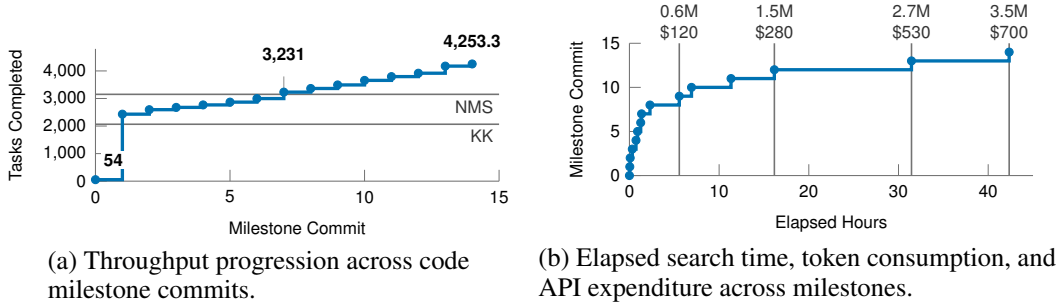}
\end{axis}
\end{tikzpicture}\par\vspace{2pt}
\figurepanelcaption{\figurepanellabel{fig:lorr:dev:throughput} Throughput progression across code milestone commits.}
\end{minipage}\hfill
\begin{minipage}[b]{0.49\linewidth}\centering
\begin{tikzpicture}
\begin{axis}[
  lorr development axis,
  scale only axis,width=0.78\linewidth,height=1.75cm,
  xlabel={Elapsed Hours},ylabel={Milestone Commit},
  xmin=0,xmax=\LorrRFiveXAxisMax,ymin=0,ymax=\LorrRFiveMilestoneAxisMax,
  xtick={0,10,20,30,40},ytick={0,5,10,15},clip=false,
]
\draw[black!55,solid,line width=0.6pt] (axis cs:5.56,0) -- (axis cs:5.56,\LorrRFiveMilestoneAxisMax);
\node[font=\scriptsize\sffamily,text=black!80,inner sep=0pt,align=center,anchor=south,yshift=1pt] at (axis cs:5.56,\LorrRFiveMilestoneAxisMax) {0.6M\\\$120};
\draw[black!55,solid,line width=0.6pt] (axis cs:16.16,0) -- (axis cs:16.16,\LorrRFiveMilestoneAxisMax);
\node[font=\scriptsize\sffamily,text=black!80,inner sep=0pt,align=center,anchor=south,yshift=1pt] at (axis cs:16.16,\LorrRFiveMilestoneAxisMax) {1.5M\\\$280};
\draw[black!55,solid,line width=0.6pt] (axis cs:31.45,0) -- (axis cs:31.45,\LorrRFiveMilestoneAxisMax);
\node[font=\scriptsize\sffamily,text=black!80,inner sep=0pt,align=center,anchor=south,yshift=1pt] at (axis cs:31.45,\LorrRFiveMilestoneAxisMax) {2.7M\\\$530};
\draw[black!55,solid,line width=0.6pt] (axis cs:42.34,0) -- (axis cs:42.34,\LorrRFiveMilestoneAxisMax);
\node[font=\scriptsize\sffamily,text=black!80,inner sep=0pt,align=center,anchor=south,yshift=1pt] at (axis cs:42.34,\LorrRFiveMilestoneAxisMax) {3.5M\\\$700};
\addplot[color=lorrBlue,const plot,solid,line width=1.25pt,no marks]
 table[x=hours,y=milestone,col sep=comma]{assets/lorr/random-05-milestones.csv};
\addplot[only marks,color=lorrBlue,mark=*,mark size=1.5pt,
 mark options={solid,fill=lorrBlue}]
 table[x=hours,y=milestone,col sep=comma]{assets/lorr/random-05-milestones.csv};
\end{axis}
\end{tikzpicture}\par\vspace{2pt}
\figurepanelcaption{\figurepanellabel{fig:lorr:dev:time} Elapsed search time, token consumption, and API expenditure across milestones.}
\end{minipage}
\caption{Policy search trajectory and computational cost of PLATM on RANDOM-05 (800 robots).}
\label{fig:lorr:dev}
\end{figure}

\newcommand{\CgarSetupRows}{%
WAREHOUSE-01 & 65 & 1,266 & 3,884 & 294 & 6 \\
WAREHOUSE-02 & 17 & 799 & 2,433 & 226 & 2 \\
WAREHOUSE-03 & 8 & 292 & 964 & 88 & 1 \\
WAREHOUSE-04 & 19 & 894 & 2,926 & 241 & 4 \\
WAREHOUSE-05 & 13 & 1,116 & 3,917 & 312 & 1 \\
WAREHOUSE-06 & 32 & 2,538 & 7,800 & 725 & 4 \\
}

\newcommand{\CgarRows}{%
WAREHOUSE-01 & 65 & 51.8 & 487.9 & 51.9 & \textbf{775.3} & \textbf{25.4} & 733.0 & 808.5 & 11.93 & $+58.9\%$\phantom{)\tnote{2}} \\
WAREHOUSE-02 & 17 & 43.5 & (127.0 & 49.6)\rlap{\smash{\tnote{2}}} & \textbf{278.6} & \textbf{2.7} & 275.0 & 282.0 & 16.39 & ($+119.4\%$)\smash{\tnote{2}} \\
WAREHOUSE-03 & 8 & 21.4 & (17.7 & 14.8)\rlap{\smash{\tnote{2}}} & \textbf{135.1} & \textbf{5.8} & 128.5 & 144.0 & 16.89 & ($+664.0\%$)\smash{\tnote{2}} \\
WAREHOUSE-04 & 19 & 49.3 & (79.8 & 24.6)\rlap{\smash{\tnote{2}}} & \textbf{269.4} & \textbf{8.7} & 254.5 & 285.0 & 14.18 & ($+237.6\%$)\smash{\tnote{2}} \\
WAREHOUSE-05 & 13 & 38.3 & (28.4 & 7.3)\rlap{\smash{\tnote{2}}} & \textbf{191.4} & \textbf{7.5} & 177.5 & 201.5 & 14.73 & ($+574.7\%$)\smash{\tnote{2}} \\
WAREHOUSE-06 & 32 & 74.2 & (257.2 & 22.5)\rlap{\smash{\tnote{2}}} & \textbf{304.8} & \textbf{12.6} & 290.0 & 320.5 & 9.52 & ($+18.5\%$)\smash{\tnote{2}} \\
}

\newcommand{\CgarCommitCount}{50}
\newcommand{\CgarCommitElapsed}{441.091111111}
\newcommand{\CgarCommitAxisMax}{55}
\newcommand{\CgarTimeAxisMax}{450}
\newcommand{\CgarActiveElapsed}{179.881418708}
\newcommand{\CgarActiveAxisMax}{190}
\newcommand{\CgarMilestoneCount}{15}
\newcommand{\CgarMilestoneAxisMax}{16}

\newcommand{\CgarScalingRuns}{800}
\newcommand{\CgarScalingMaxFleet}{100}
\newcommand{\CgarScalingBaseline}{441}
\newcommand{\CgarScalingBaselineFleet}{65}
\newcommand{\CgarScalingMatchBaselineFleet}{26}
\newcommand{\CgarScalingMatchBaselineMean}{441.6}
\newcommand{\CgarScalingMatchBaselineShare}{40}
\newcommand{\CgarScalingMatchKnappFleet}{30}
\newcommand{\CgarScalingMatchKnappMean}{492.8}
\newcommand{\CgarScalingKnappMean}{487.9}
\newcommand{\CgarScalingPeakFleet}{68}
\newcommand{\CgarScalingPeakMean}{784.2}
\newcommand{\CgarScalingLastMean}{415.8}

\newcommand{\CgarSummaryRows}{%
WAREHOUSE-01 & 65 & 487.9 & 51.9\tnote{a} & 775.3 & 25.4 & $+58.9\%$ \\
WAREHOUSE-02 & 17 & 127.0 & 49.6\tnote{b} & 278.6 & 2.7 & $+119.4\%$ \\
WAREHOUSE-03 & 8 & 17.7 & 14.8\tnote{b} & 135.1 & 5.8 & $+664.0\%$ \\
WAREHOUSE-04 & 19 & 79.8 & 24.6\tnote{b} & 269.4 & 8.7 & $+237.6\%$ \\
WAREHOUSE-05 & 13 & 28.4 & 7.3\tnote{b} & 191.4 & 7.5 & $+574.7\%$ \\
WAREHOUSE-06 & 32 & 257.2 & 22.5\tnote{b} & 304.8 & 12.6 & $+18.5\%$ \\
}

\section{Case Study 2: Industrial Warehouse}
\label{sec:warehouse}

Our second case study investigates whether PSCS can synthesize an entirely new algorithm for physically constrained, real-world LMAPF environments.
We evaluate our approach on high-fidelity simulations of operational KNAPP industrial warehouses, wherein an autonomous mobile fleet operates under continuous 3D vehicle geometries, kinematic lane constraints, and asynchronous execution.

\subsection{Experiment Setup}
\label{sec:warehouse:problem_setup}

The warehouse LMAPF setting generalizes the formulation in Section~\ref{sec:problem_setup} to 3D continuous space.
Figure~\ref{fig:test2} illustrates elements of this continuous warehouse layout.
A \emph{storage node} is the ground-level interface to a vertical column of storage locations.
A configuration vertex $u \in V$ represents a robot's 3D pose (position and orientation) and \emph{service status}, with the robot located either on the warehouse floor or elevated at a storage node.
The service status records the stage of a docking or material-handling operation, such as undocked, docked, retrieved, stored, loaded, or unloaded.
Directed edges $e = (u,v) \in E$ encompass forward and backward lane translations, in-place rotations, lateral docking maneuvers at storage nodes, and vertical lift transitions.
Edge durations $\delta(e)$ are governed by physical kinematics, with bounded floor velocities and accelerations, rotational speed limits, and vertical movement speed limits.
An errand target $W$ comprises the vertices at a storage node or picking station with a few seconds of service time $\sigma > 0$.

Spatial conflicts are evaluated via continuous volumetric occupancy.
Let $S(z) \subset \mathbb{R}^3$ denote the swept 3D bounding volume occupied by an agent at vertex pose or along edge transition $z \in V \cup E$.
The geometric conflict mapping is defined as $\Gamma(z) = \{z' \in V \cup E \mid S(z) \cap S(z') \neq \varnothing\}$.
This relation captures spatial interference that grid abstractions cannot model, such as wide vehicle sweeps into adjacent corridors during turning maneuvers.

Orders arrive via an exogenous stochastic stream at approximately 1,800 orders per hour.
Each order consists of an alternating sequence of pickup and drop-off errands between storage nodes and picking stations.
Simulations run for a horizon of $T = 7{,}200$\,s, with throughput measured in completed orders per hour.
Algorithms must make decisions faster than real time.

The search was initiated using Codex 5.6 Sol and subsequently Codex Astra (upon its release) within a command-line environment equipped with the warehouse simulator.
Auxiliary models Fable 5.1, Fable 5, and Opus 5 were occasionally consulted for code review and diagnostic feedback at the discretion of the primary coding agents.
The search was seeded solely with a design specification: the agent was required to build an algorithm that, under stated assumptions, provably completes, in finite time, every accepted goal that is not withdrawn or replaced, while leaving architectural freedom for throughput heuristics to develop.
The user directed the coding agent to prototype on WAREHOUSE-01 first and later generalize across the remaining five warehouse environments (Figure~\ref{fig:cgar:maps}).
The objective was to maximize order throughput, subject to continuous geometric collision avoidance, legal order assignments, and real-time execution, with no explicit memory constraints.

\begin{figure}[t]
\centering
\includegraphics[width=0.55\linewidth]{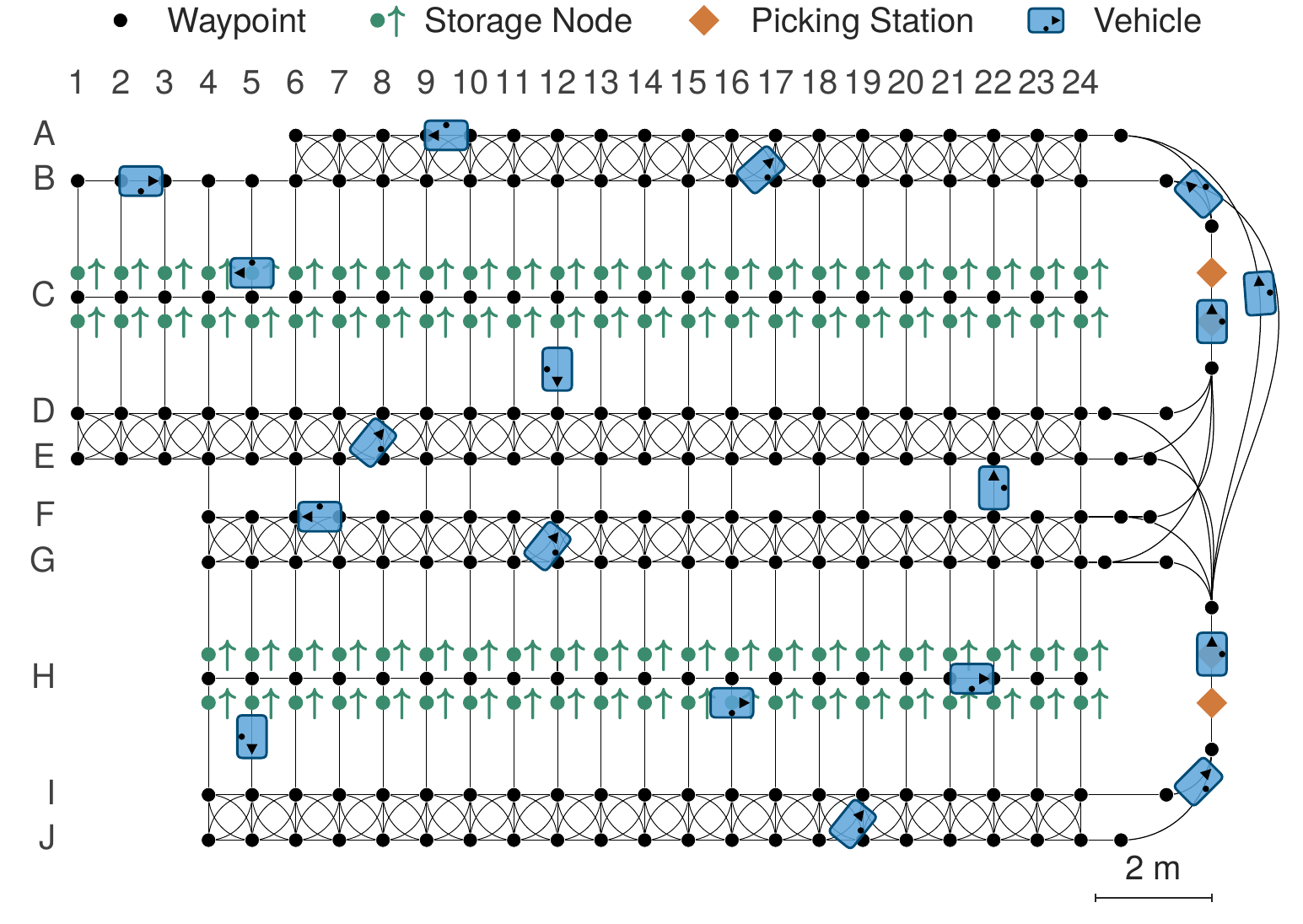}
\caption{An illustrative warehouse layout.
Black lanes connect waypoints, green markers denote storage nodes, and orange diamonds denote picking stations.}
\label{fig:test2}
\end{figure}

\begin{figure}[t]
\centering
\renewcommand{\maplabelfont}{\normalfont\sffamily\fontsize{8}{9.5}\selectfont}
\input{assets/cgar/warehouse-maps-panels.tex}
\WarehouseMapPanels{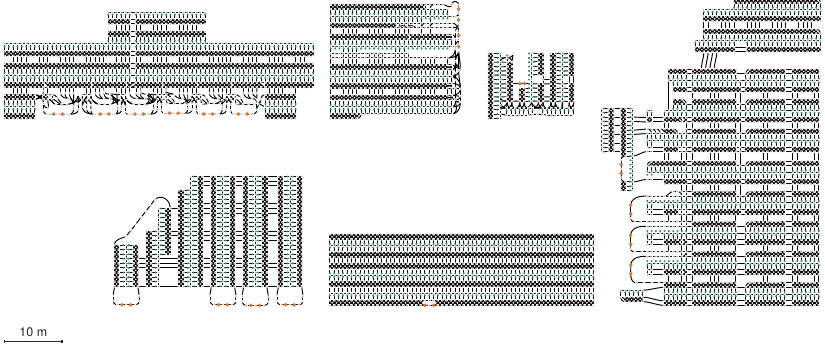}
\caption{The six warehouse layouts.
Maps are mirrored, and some rotated, for anonymity.}
\label{fig:cgar:maps}
\end{figure}

\subsection{Results and Analysis}
\label{sec:warehouse:results}

\begin{table}[t]
\centering
\caption{Orders completed per hour by CGAR and KNAPP, as mean $\pm$ s.d.\ over eight paired trials.}
\label{tab:cgar:final}
\small
\setlength{\tabcolsep}{3pt}
\begin{threeparttable}
\begin{tabular}{@{}l rr w{r}{2.8em}@{\,\(\pm\)\,}w{l}{2.8em} r@{\,\(\pm\)\,}l rrr r@{}}
\toprule
 & & & \multicolumn{2}{c}{KNAPP} & \multicolumn{5}{c}{CGAR} & \\
\cmidrule(lr){4-5}\cmidrule(lr){6-10}
 & \multicolumn{1}{c}{$n$\rlap{\smash{\tnote{1}}}} & \multicolumn{1}{c}{$d$\rlap{\smash{\tnote{1}}}} & Mean & S.D. & Mean & S.D. & Min & Max & \multicolumn{1}{c}{$\mu/n$\rlap{\smash{\tnote{1}}}} & Gain \\
\midrule
\CgarRows
\bottomrule
\end{tabular}
\begin{tablenotes}[flushleft]
\footnotesize\raggedright
\item[1] $n$ is the number of robots, $d$ the mean shortest path length of an order in meters, and $\mu/n$ the CGAR mean throughput per robot.
\item[2] Values in parentheses are not optimized, as KNAPP has not tuned these maps for their method.
\end{tablenotes}
\end{threeparttable}
\end{table}

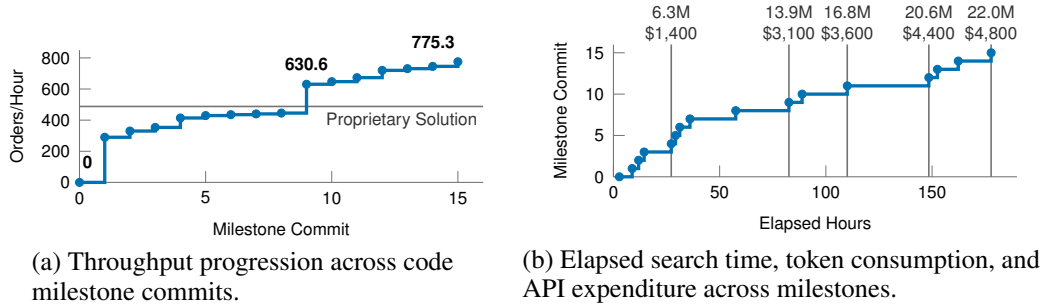
\begin{figure}[t]
\centering
\begin{minipage}[b]{0.49\linewidth}\centering
\begin{tikzpicture}
\begin{axis}[
  lorr development axis,
  scale only axis,width=0.78\linewidth,height=1.75cm,
  xlabel={Milestone Commit},ylabel={Orders/Hour},
  xmin=0,xmax=\CgarMilestoneAxisMax,ymin=0,ymax=850,
  xtick={0,5,10,15},ytick={0,200,400,600,800},clip=false,
]
\draw[black!55,solid,line width=0.6pt]
  (axis cs:0,\CgarScalingKnappMean) -- (axis cs:\CgarMilestoneAxisMax,\CgarScalingKnappMean);
\node[font=\scriptsize\sffamily,text=black!80,fill=white,inner xsep=2pt,
  inner ysep=0pt,anchor=north east,yshift=-2pt]
  at (axis cs:\CgarMilestoneAxisMax,\CgarScalingKnappMean) {Proprietary Solution};
\addplot[color=lorrBlue,const plot,solid,line width=1.25pt,no marks]
  table[col sep=comma,x=milestone,y=throughput]{assets/cgar/cgar-checkpoint-commits.csv};
\addplot[only marks,color=lorrBlue,mark=*,mark size=1.5pt,
  mark options={solid,fill=lorrBlue}]
  table[col sep=comma,x=milestone,y=measured]{assets/cgar/cgar-checkpoint-commits.csv};
\addplot[only marks,color=lorrBlue,mark=*,mark size=1.5pt,
  mark options={solid,fill=lorrBlue}]
  table[col sep=comma,x=milestone,y=estimated]{assets/cgar/cgar-checkpoint-commits.csv};
\node[lorr value note,anchor=south west] at (axis cs:0,0) {0};
\node[lorr value note,anchor=south] at (axis cs:9,630.625) {630.6};
\node[lorr value note,anchor=south east] at (axis cs:15,775.3125) {775.3};
\end{axis}
\end{tikzpicture}\par\vspace{2pt}
\figurepanelcaption{\figurepanellabel{fig:cgar:dev:throughput} Throughput progression across code milestone commits.}
\end{minipage}\hfill
\begin{minipage}[b]{0.49\linewidth}\centering
\begin{tikzpicture}
\begin{axis}[
  lorr development axis,
  scale only axis,width=0.78\linewidth,height=1.75cm,
  xlabel={Elapsed Hours},ylabel={Milestone Commit},
  xmin=0,xmax=\CgarActiveAxisMax,ymin=0,ymax=\CgarMilestoneAxisMax,
  xtick={0,50,100,150},ytick={0,5,10,15},clip=false,
]
\draw[black!55,solid,line width=0.6pt] (axis cs:27.23,0) -- (axis cs:27.23,\CgarMilestoneAxisMax);
\node[font=\scriptsize\sffamily,text=black!80,inner sep=0pt,align=center,anchor=south,yshift=1pt] at (axis cs:27.23,\CgarMilestoneAxisMax) {6.3M\\\$1,400};
\draw[black!55,solid,line width=0.6pt] (axis cs:82.62,0) -- (axis cs:82.62,\CgarMilestoneAxisMax);
\node[font=\scriptsize\sffamily,text=black!80,inner sep=0pt,align=center,anchor=south,yshift=1pt] at (axis cs:82.62,\CgarMilestoneAxisMax) {13.9M\\\$3,100};
\draw[black!55,solid,line width=0.6pt] (axis cs:110.08,0) -- (axis cs:110.08,\CgarMilestoneAxisMax);
\node[font=\scriptsize\sffamily,text=black!80,inner sep=0pt,align=center,anchor=south,yshift=1pt] at (axis cs:110.08,\CgarMilestoneAxisMax) {16.8M\\\$3,600};
\draw[black!55,solid,line width=0.6pt] (axis cs:148.48,0) -- (axis cs:148.48,\CgarMilestoneAxisMax);
\node[font=\scriptsize\sffamily,text=black!80,inner sep=0pt,align=center,anchor=south,yshift=1pt] at (axis cs:148.48,\CgarMilestoneAxisMax) {20.6M\\\$4,400};
\draw[black!55,solid,line width=0.6pt] (axis cs:177.80,0) -- (axis cs:177.80,\CgarMilestoneAxisMax);
\node[font=\scriptsize\sffamily,text=black!80,inner sep=0pt,align=center,anchor=south,yshift=1pt] at (axis cs:177.80,\CgarMilestoneAxisMax) {22.0M\\\$4,800};
\addplot[color=lorrBlue,const plot,solid,line width=1.25pt,no marks]
 table[x=active_hours,y=milestone,col sep=comma]{assets/cgar/cgar-checkpoint-commits.csv};
\addplot[only marks,color=lorrBlue,mark=*,mark size=1.5pt,
 mark options={solid,fill=lorrBlue}]
 table[x=active_hours,y=milestone,col sep=comma]{assets/cgar/cgar-checkpoint-commits.csv};
\end{axis}
\end{tikzpicture}\par\vspace{2pt}
\figurepanelcaption{\figurepanellabel{fig:cgar:dev:time} Elapsed search time, token consumption, and API expenditure across milestones.}
\end{minipage}
\caption{Policy search trajectory and computational cost of CGAR on WAREHOUSE-01 (65 robots).}
\label{fig:cgar:dev}
\end{figure}

We benchmark our synthesized policy, named Core-Graph-Anchored Routing (CGAR), against an optimized proprietary production baseline.
As in Section~\ref{sec:lorr}, all evaluations execute on identical hardware, pairing eight trials per map with fixed initial poses and shared random seeds across methods for order arrivals (CGAR's decisions are deterministic).

Table~\ref{tab:cgar:final} reports throughput across eight paired two-hour evaluation runs for all six warehouse facilities.
CGAR outperforms the proprietary solution by 58.9\% on WAREHOUSE-01 and achieves higher throughput across the remaining five layouts where the baseline was uncalibrated.
The proprietary solution relies on extensive manual calibration of lane directions, traversal costs, and vertex-level stopping rules for each specific floor plan, limiting its transferability across facilities.
In contrast, because CGAR operates on topology-agnostic coordination, a single synthesized policy transfers directly across all layouts without any site-specific recalibration.
Across fleet sizes from 8 to 65 robots and mean route lengths from 21 to 74\,m, CGAR consistently sustains 9.5 to 16.9 orders per hour per robot.

Figure~\ref{fig:cgar:dev} illustrates how the coding agent synthesized and refined CGAR on WAREHOUSE-01.
The initial seed implementation enforced the progress guarantee too rigidly, actuating robots in near-sequential fashion, completing almost zero orders.
The agent then unlocked concurrent execution by certifying simultaneous transitions that preserved a return path to a position where progress is guaranteed ($\sim$350 orders/hour).
It next resolved stop-and-go patterns and fleet freezes by stabilizing lane admission decisions ($\sim$440 orders/hour).
Throughput then plateaued near 445 orders/hour because robots took unnecessary detours around temporarily occupied corridors.
The upgraded coding models (Astra and Fable 5.1) resolved this by introducing route commitments and fair queue release (630.6 orders/hour).
Finally, the agent improved long-distance yielding maneuvers and refined order dispatching (775.3 orders/hour).
Across the entire synthesis process, the search utilized 177.8 elapsed development hours and consumed 22.0M output tokens ($\approx$\$4,800).

The resulting CGAR architecture decouples formal progress certification from heuristic throughput optimization:
At its foundation, an invariant ensures that every admitted motion keeps a recovery plan back to the core; under the assumptions of Theorem~\ref{thm:cgar} (Appendix~\ref{app:cgar}), this guarantees that every accepted goal that is not withdrawn or replaced is completed in finite time.
Operating above this safety layer, throughput heuristics optimize traffic flow without the risk of inducing deadlock under the same assumptions.
Specifically, the planner treats occupied lanes as temporary delays rather than static obstacles whenever the current occupant has committed to departure, avoiding unnecessary detouring loops.
Concurrently, the task dispatcher assigns orders based on predicted lane congestion and trip duration, preventing bottlenecks before traffic accumulates.

\section{Discussion and Conclusion}
\label{sec:discussion}

Across two challenging case studies, PSCS demonstrates that agentic code synthesis can discover algorithms that significantly outperform mature, human-engineered baselines.
On the LoRR 2024 benchmark, PLATM surpasses the top two competition entries across all five instances, achieving throughput gains of up to 35.0\% in extreme congestion.
In simulated industrial warehouses, CGAR establishes a new class of multi-agent routing algorithms, delivering a 58.9\% throughput increase over the proprietary production baseline on WAREHOUSE-01 while operating across all six facilities within a single, uncalibrated configuration.
Together, these findings establish PSCS as a viable, highly effective methodology for complex combinatorial algorithm design.

Both the search methodology and the resulting policies exhibit strong generalization capabilities.
Because competitive LoRR entries rely on map-specific calibration, we evaluate whether PLATM transfers effectively to unseen layouts when provided comparable tuning.
As detailed in Appendix~\ref{app:heldout:lorr}, recalibrating PLATM on a newly generated layout yields throughput within 3.9\% of the competition benchmark across all fleet densities.
In the industrial warehouse domain, CGAR exhibits zero-shot generalization: evaluated on the unseen synthetic layout WAREHOUSE-00 without any parameter adjustment, it sustains 16.98 orders per hour per robot, matching the operational efficiency observed on production facilities of comparable scale (Appendix~\ref{app:heldout:warehouse}).

A primary operational constraint of PSCS is the computational and financial cost of empirical evaluation.
The synthesis of PLATM required 42.3 wall-clock hours and approximately \$700 in API usage, whereas CGAR required 177.8 active development hours and approximately \$4{,}800.
This expenditure is driven largely by the simulation bottleneck: each candidate policy must be evaluated over extended temporal horizons and multiple stochastic realizations before returning informative performance feedback to the coding agent.
Accelerating policy evaluation is therefore a critical research direction, motivating techniques such as temporal parallelization~\citep{farias2025speeding}, importance sampling, and efficient surrogate simulators.

A notable trait of policies synthesized through PSCS is their accumulated software complexity.
Whereas human engineers prioritize modularity, architectural symmetry, and maintainability, the coding agent optimizes strictly for empirical performance.
As the agent resolves successive edge cases and failure modes, the codebase accumulates tightly coupled heuristics, yielding monolithic implementations that are challenging for developers to audit or modify.
Future iterations of PSCS should incorporate software engineering best practices into the search objective, such as penalizing cyclomatic complexity or introducing periodic refactorization.

We also note a methodological consideration in our experimental execution: the policies in Sections~\ref{sec:lorr} and~\ref{sec:warehouse} were developed through organic, human-in-the-loop steering rather than a fully autonomous, hardcoded search harness.
It remains an open question to what extent rigidly encapsulating search logic within a static agentic scaffold risks overfitting to the limitations of a specific model release.
Systematically characterizing the optimal boundary between programmatic search scaffolding and open-ended, model-driven steering represents an important direction for future inquiry.

Looking forward, PSCS represents a compelling paradigm for combinatorial algorithm design in verifiable domains.
Like traditional software engineering, PSCS produces source code that incurs zero neural inference latency at runtime and admits formal safety verification under stated assumptions, as demonstrated by the progress certificates of CGAR (Theorem~\ref{thm:cgar}).
Simultaneously, like deep learning, it scales with compute and optimizes directly against empirical loss landscapes through automated iteration.
By uniting the formal verifiability of symbolic programs with the scalability of data-driven optimization, PSCS offers a practical alternative to both manual engineering and end-to-end learning.
In this paradigm, code serves as the new weights.

\subsection*{AI Use Statement}

Generative AI is part of the subject of this paper.
Coding agents wrote, tested, and diagnosed the policy code of both case studies.
Codex Astra did this work for PLATM.
For CGAR, Codex 5.6 Sol and Codex Astra did most of it, Fable 5.1 contributed two mechanisms, and Fable 5 and Opus 5 gave feedback (Sections~\ref{sec:lorr:problem_setup} and~\ref{sec:warehouse:problem_setup}).
The authors set the objectives and constraints and directed the work, and the final results passed the validity checks of Section~\ref{sec:method}.
We also used AI assistants to help draft and edit the text, to produce figures and tables from our experiment records, and to check citations against their primary sources.
The authors reviewed all AI-assisted text, figures, and code.
We take responsibility for the final content of this paper, including all claims.

\bibliography{paper/references}
\bibliographystyle{iclr2027_conference}

\clearpage
\appendix
\ifpapertmlr\else\raggedbottom\fi
\section{The CGAR Algorithm}
\label{app:cgar}

CGAR is briefly described in Section~\ref{sec:warehouse}.
This appendix gives the full algorithm (Section~\ref{app:alg}).
It proves that every execution is collision-free and that every goal the dispatcher does not withdraw or replace is completed in finite time (Section~\ref{app:proof}).

\subsection{Overview}
\label{app:overview}

CGAR is a PIBT-style planner ~\citep{okumura2022priority} for robots with continuous geometry that move asynchronously.
The dispatcher assigns orders to robots and supplies their successive pickup and drop-off goals.
CGAR routes each robot to its current goal using two layers.
The \emph{throughput layer} proposes motions using priority inheritance and heuristics.
The \emph{safety layer} decides which proposals are committed, and supplies fallback motions when no proposal can be committed.
This follows the Simplex architecture~\citep{sha2001using}: an unverified heuristic may run freely as long as a safe fallback remains available.

During initialization, CGAR selects ground poses called \emph{core nodes} and connects them by \emph{core edges}.
A core edge corresponds to a fixed sequence of physical motions, which we call a \emph{core path}.
Together, core nodes and edges form the \emph{core graph}.
The robots are \emph{on the core graph} when every robot occupies a different core node.
The safety layer rests on one invariant: \emph{at every moment, all robots can be brought back to the core graph}.
CGAR commits a motion only if two conditions hold.
First, the motion is collision-free.
Second, after the motion is committed, there exists a \emph{recovery plan}: a plan that brings all robots back to the core graph.
These two checks maintain the invariant.
Once the robots are on the core graph, and there are more core nodes than robots, any robot can reach any core node by shifting other robots along a vacancy, like a sliding puzzle.

CGAR prioritizes the robot with the oldest active goal.
If that goal remains active for $D$ seconds after becoming the oldest, CGAR switches to the fallback the next time it selects motions.
It pauses normal routing, lets ongoing motions finish, and executes the recovery plan to return all robots to the core graph.
It then moves the selected robot to its goal, shifting other robots into vacant core nodes as needed.

\subsection{Algorithm}
\label{app:alg}

We use the configuration graph $G=(V,E)$ and the conflict map $\Gamma$ of Sections~\ref{sec:problem_setup} and~\ref{sec:warehouse:problem_setup}.
Each vertex in $V$ represents a robot configuration defined by the robot's position, orientation, and service status.
Each edge in $E$ represents a feasible motion between configurations, such as lane traversal, in-place rotation, or vertical movement.
Execution is asynchronous: robots start and finish motions at different times.
Actual traversal times may differ from planning estimates because execution accounts for acceleration and braking.
CGAR selects motions when a motion or service finishes, when the dispatcher assigns, replaces or withdraws a goal, or at a periodic controller tick (1 s by default).

A path is a short sequence of motions of one robot, and each robot executes at most one committed path at a time.
A new path is committed only if two conditions hold.
\begin{itemize}
  \item[(S1)] The volume it sweeps, the union of $S(e)$ over its motions $e$ (Section~\ref{sec:warehouse:problem_setup}), is disjoint from the volume swept by every other path in progress, so the two cannot collide, whatever their timing.
  \item[(S2)] It does not intersect $S(u)$ for any robot standing still at a vertex $u$.
\end{itemize}

The \emph{core graph} is specified by the triple \((B, H, P)\).
The set $B\subseteq V$ consists of ground poses of \(G\) (a position on the floor and an orientation), called \emph{core nodes}, and \(H\) and \(P\) are directed graphs on $B$.
\begin{itemize}
  \item The \emph{support graph} $H=(B,E_H)$.
    Each \emph{core edge} $u\to v$ has an associated \emph{core path}: a fixed path from $u$ to $v$ in \(G\) whose edges are lane traversals and in-place rotations.
  \item The \emph{promotion graph} $P=(B,E_P)$, with $E_P\subseteq E_H$.
\end{itemize}
Here, $E_H$ and $E_P$ are the edge sets of $H$ and $P$, respectively.
The two graphs have different roles.
Core edges of $H$ are used to move robots out of the way, and core edges of $P$ are used to move a robot toward its goal.
$P$ may leave out core edges of $H$ that are safe for the first use but not for the second (condition (K5) below).

During initialization, CGAR builds the core graph automatically from the map and checks the following five conditions for a fleet of $n$ robots.
\begin{itemize}
  \item[(K1)] \emph{Separation.} Robots standing at distinct core nodes do not overlap: $S(b)\cap S(b')=\emptyset$ for any two distinct core nodes $b,b'\in B$.
  \item[(K2)] \emph{Clear core paths.} The core path for every core edge $u\to v$ is free of conflicts with any core node other than $u$ and $v$.
  \item[(K3)] \emph{Connectivity.} $H$ and $P$ are strongly connected.
  \item[(K4)] \emph{Vacancy.} The core graph contains more nodes than robots: $|B|\ge n+1$.
    Thus, whenever all robots are on the core graph, at least one core node is vacant.
  \item[(K5)] \emph{Room to advance.} For every core edge $u\to v$ of $P$, at least $n$ core nodes can be reached from $v$ in $H$ without passing through $u$.
\end{itemize}
The last condition ensures that while a robot stands at $u$, at least one core node reachable from $v$ is vacant.
The robots between $v$ and that core node can then shift toward it one at a time without passing through $u$, which frees $v$ for the robot at $u$.

The checks passed on all six tested warehouses.
Should any check fail, CGAR could not be applied.
Figure~\ref{fig:core} shows a core graph projected onto WAREHOUSE-00's ground network.

\begin{figure}[t]
    \centering
    \includegraphics[width=0.55\linewidth]{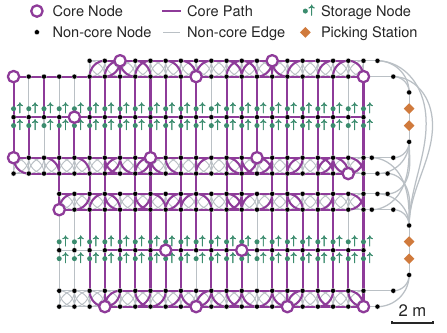}
    \caption{The core graph of WAREHOUSE-00, where $H=P$.
      Hollow purple circles mark its 17 core nodes; purple curves show core paths projected onto the ground network, omitting headings and rotations.}
    \label{fig:core}
\end{figure}

A \emph{recovery plan} is a sequence of steps that returns all robots to the core graph after all currently committed paths finish.
Each step is a set of paths, at most one per robot, that satisfy (S1) and (S2) with respect to each other and to every robot standing still, so they can run at the same time.
CGAR always stores a recovery plan.
CGAR commits new candidate paths only if a recovery plan exists from the robots' positions after all previously committed paths and the candidate paths finish.
CGAR searches for such a plan with a bounded search and checks it by simulating every step.
If no plan is found, the new paths are rejected.

The dispatcher assigns each robot one goal at a time, consisting of a target vertex to reach and, if required, a service to perform.
Each newly assigned goal receives a ticket from an increasing counter, so smaller tickets correspond to older goals.

At each routing decision, every robot proposes candidate paths, such as prefixes of its shortest route, pushes of a blocking robot as in PIBT, and moves along the core graph.
CGAR visits the robots in ticket order, starting with the robot with the oldest active goal, and adds each robot's best candidate that satisfies (S1) and (S2) with respect to the paths chosen so far.
It commits the chosen paths if a recovery plan exists, and otherwise removes paths until one does.
If nothing can be committed, it uses the steps of the fallback described below, without stopping other robots.

The throughput layer employs heuristics that govern which candidate paths are proposed and how they are ranked:
\begin{itemize}
  \item A robot blocked only by traffic that will clear by itself waits instead of detouring;
  \item Robots blocked at a service location back off to the core graph and retry;
  \item The robot with the oldest active goal can be overtaken only a bounded number of times before other robots must yield to it;
  \item Helper robots making way for this robot may move in parallel when this does not delay it.
\end{itemize}
These heuristics change only which paths are proposed and in what order.
Every committed path still passes the same checks.

When a goal becomes the oldest active goal, its fallback deadline is set to $D$ seconds later.
At the first routing decision at or after this deadline, if the goal remains active, CGAR pauses normal routing and starts the following fallback for the robot with that goal.
No new work outside the fallback is committed until the goal is completed, withdrawn or replaced.
\begin{enumerate}
  \item \emph{Drain.} Wait for committed motions to finish.
  \item \emph{Recover.} Execute the stored recovery plan step by step.
  \item \emph{Promote.} Move the robot with the oldest active goal along a shortest path in $P$ to the \emph{entry node} of its goal, a core node from which it can reach its goal and return without conflicting with any other core node.
    For each core edge $u\to v$ of this path, the robots between $v$ and a vacant core node first move one core node forward, one at a time, as described after (K5).
    This frees $v$ for that robot to move from $u$ to $v$.
  \item \emph{Serve.} Move the robot from its entry node to its goal, and keep the path back to the entry node as the recovery plan.
\end{enumerate}
If the fallback cannot construct a step, execution stops with an error rather than returning to heuristic routing.
Figure~\ref{fig:cgar-flow} summarizes one decision.

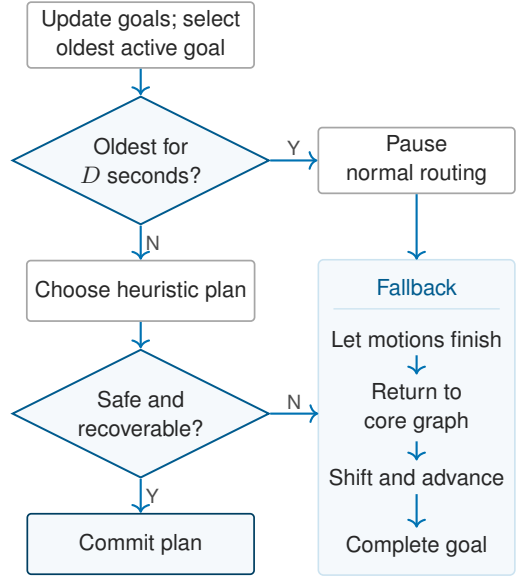
\begin{figure}[t]
\centering
\begin{tikzpicture}[
  x=1cm, y=1cm,
  font=\normalfont\sffamily\fontsize{9}{11}\selectfont,
  every node/.style={text=black!85, align=center, inner sep=3pt},
  action/.style={rectangle, draw=black!35, fill=white, line width=0.65pt,
                 rounded corners=1.5pt, minimum width=3.0cm, minimum height=8mm},
  gate/.style={diamond, aspect=2, minimum width=3.4cm, minimum height=1.7cm,
               draw=lorrBlue!80!black,
               fill=lorrBlue!4, line width=0.75pt, inner sep=1.5pt},
  link/.style={draw=lorrBlue, line width=0.8pt, rounded corners=3pt},
  flow/.style={link, -{Computer Modern Rightarrow},
               shorten <=0pt, shorten >=0pt},
  branch/.style={font=\normalfont\sffamily\fontsize{8}{10}\selectfont,
                 text=black!70, inner sep=1.5pt},
  phase/.style={inner xsep=1pt, inner ysep=3pt},
]
\node[action] (update) at (0,0)
  {Update goals; select\\oldest active goal};
\node[gate, below=3.5mm of update] (deadline)
  {Oldest for\\$D$ seconds?};
\node[action, below=4.5mm of deadline] (cand)
  {Choose heuristic plan};
\node[gate, below=3.5mm of cand] (check)
  {Safe and\\recoverable?};
\node[action, below=4.5mm of check,
      draw=lorrBlue!55!black, fill=lorrBlue!4]
  (commit) {Commit plan};

\draw[flow] (update.south) -- (deadline.north);
\draw[flow] (deadline.south) -- node[branch,right]{N} (cand.north);
\draw[flow] (cand.south) -- (check.north);
\draw[flow] (check.south) -- node[branch,right]{Y} (commit.north);

\node[action, minimum width=2.6cm] (stop) at (3.65,0 |- deadline.center)
  {Pause\\normal routing};
\draw[flow] (deadline.east) -- node[branch,above]{Y} (stop.west);

\coordinate (fbtop) at (stop.center |- cand.north);
\coordinate (fbbottom) at (stop.center |- commit.south);
\coordinate (fbwest) at ([xshift=-1.3cm]fbtop);
\begin{scope}[on background layer]
  \path[fill=lorrBlue!4, draw=lorrBlue!20, line width=0.6pt,
        rounded corners=2pt]
    (fbwest) rectangle ([xshift=1.3cm]fbbottom);
\end{scope}
\node[text=lorrBlue!80!black, anchor=north, inner sep=0pt]
  at ([yshift=-2.5mm]fbtop) {Fallback};
\draw[lorrBlue!20, line width=0.5pt]
  ([xshift=-1.1cm,yshift=-6.5mm]fbtop) --
  ([xshift=1.1cm,yshift=-6.5mm]fbtop);
\node[phase] (drain) at ($(fbtop)!0.25!(fbbottom)$)
  {Let motions finish};
\node[phase] (recover) at ($(fbtop)!0.47!(fbbottom)$)
  {Return to\\core graph};
\node[phase] (promote) at ($(fbtop)!0.69!(fbbottom)$)
  {Shift and advance};
\node[phase] (serve) at ($(fbtop)!0.91!(fbbottom)$)
  {Complete goal};
\foreach \from/\to in {drain/recover,recover/promote,promote/serve}
  \draw[flow] (\from.south) -- (\to.north);

\draw[flow] (stop.south) -- (fbtop);
\draw[flow] (check.east) -- node[branch,above]{N} (fbwest |- check.east);
\end{tikzpicture}
\caption{CGAR motion selection.
Committed motions must be collision-free and preserve a recovery plan.
If a goal remains the oldest active goal for $D$ seconds, normal routing pauses for the fallback.}
\label{fig:cgar-flow}
\end{figure}

\subsection{Proof of Safety and Goal Completion}
\label{app:proof}

Safety follows from the swept-volume checks, while induction over routing decisions establishes that a recovery plan remains available.
For progress, the deadline-triggered fallback completes the oldest active goal in finite time using vacancy shifts, and ticket ordering extends this guarantee to every goal not withdrawn or replaced.

Besides (K1)--(K5), which are checked during initialization, the proof uses three assumptions.
\begin{itemize}
  \item[(E1)] \emph{Finite motions.} Every motion or service, once started, finishes in finite time.
  \item[(E2)] \emph{Recoverable start.} The robots start at non-overlapping positions with a recovery plan, for example all on the core graph.
  \item[(E3)] \emph{Service coverage.} Every goal has an entry node: a core node $b$ with a path from $b$ to the goal and back that conflicts with no other core node.
    \footnote{On four of the six warehouses, (E3) holds for every goal.
    On WAREHOUSE-01 and WAREHOUSE-03, a few storage nodes (16 and 5, respectively) satisfy (E3) only if one or two additional core nodes are kept vacant.
    The argument extends to this case by first vacating those core nodes, but to keep the proof simple we state it under (E3).}
\end{itemize}
The proof also takes the runtime checks of (S1), (S2) and of the recovery plan to be correct as implemented.

The first two lemmas establish safety and preservation of a recovery plan.

\begin{lemma}[Safe Commitment]
\label{lem:safe}
Suppose the robots start at non-overlapping positions, and every committed path satisfies (S1) and (S2) when it is committed.
Then no two robots ever overlap, however long each motion takes.
\end{lemma}

\begin{proof}
A robot never leaves the volume swept by its last committed path, which includes the path's end.
By (S1), this volume is disjoint from that of every path in progress when the path was committed, and by (S2), from every robot standing still at that time.
Paths committed later are checked against this one, or against the robot standing at its end, in the same way.
\end{proof}

\begin{lemma}[Recovery Is Always Possible]
\label{lem:invariant}
CGAR always stores a recovery plan.
\end{lemma}

\begin{proof}
Initially, a recovery plan exists by (E2).
We show that every decision keeps one.
Between decisions, robots only carry out committed paths, so the stored plan stays valid.
At a decision, CGAR commits one of the following.
\begin{itemize}
  \item While normal routing is not paused, paths are committed only if a recovery plan exists afterwards.
  \item Recover executes the next step of the stored plan.
    The remaining steps form a recovery plan.
  \item Promote moves one robot along a core path while all robots are on the core graph.
    By (K1) and (K2), all robots are on the core graph again once it finishes, so no further recovery is needed.
  \item Serve moves the robot with the oldest active goal from its entry node to its goal while all other robots are on the core graph.
    By (E3), this conflicts with no other core node, and the path back to the entry node is a recovery plan.
\end{itemize}
Drain commits no new paths.
\end{proof}

We next show that the fallback completes its goal in finite time.

\begin{lemma}[Vacancy Shift]
\label{lem:shift}
Let the robots be on the core graph, let robot $r$ stand at core node $u$, and let $u\to v$ be a core edge of $P$.
Then at most $n-1$ other robots can each shift one core node along a path in $H$ that avoids $u$, after which $r$ moves to $v$.
This executes at most $n$ core paths, one at a time, and the robots are on the core graph after each core path.
\end{lemma}

\begin{proof}
By (K5), at least $n$ core nodes are reachable from $v$ in $H$ without passing through $u$.
At most $n-1$ robots other than $r$ stand on them, so at least one of them is vacant.
Let $h$ be a vacant one closest to $v$, and take a shortest path $v=q_0,\dots,q_k=h$ in $H$ that avoids $u$.
Every core node on this path before $h$ is occupied, so $k\le n-1$.
For $j=k,\dots,1$, if a robot stands at $q_{j-1}$, move it to $q_j$, which is vacant at that point.
Then $v$ is vacant, and $r$ moves from $u$ to $v$.
This uses at most $k+1\le n$ core paths.
Each core path moves one robot while the others stand on core nodes, so by (K1) and (K2) it causes no conflict and ends with the robots on the core graph.
\end{proof}

\begin{lemma}[The Fallback Completes the Goal]
\label{lem:fallback}
Once CGAR pauses normal routing for the fallback to complete a goal, that goal is completed in finite time unless it is withdrawn or replaced.
\end{lemma}

\begin{proof}
New goals receive larger tickets, so the goal being served by the fallback remains the oldest active goal until it ends.
Only fallback steps are committed while normal routing is paused, and each step finishes in finite time by (E1).
\begin{itemize}
  \item \emph{Drain:} all committed paths finish.
  \item \emph{Recover:} the stored recovery plan (Lemma~\ref{lem:invariant}) brings all robots back to the core graph.
  \item \emph{Promote:} by (K3), the entry node is at most $|B|-1$ core edges of $P$ away, so by Lemma~\ref{lem:shift} at most $(|B|-1)\,n$ core paths bring the robot with the oldest active goal there.
  \item \emph{Serve:} by (E3), the path from the entry node to the goal conflicts with no other core node, so the robot reaches its goal and performs the service.
\end{itemize}
\end{proof}

Together, these lemmas yield the main guarantee, as stated in the following theorem:

\begin{theorem}
\label{thm:cgar}
Under (K1)--(K5) and (E1)--(E3), every execution of CGAR is collision-free, a recovery plan exists at every decision, and every goal is completed in finite time unless the dispatcher withdraws or replaces it.
\end{theorem}

\begin{proof}
We first show safety.
Paths committed in normal operation are checked for (S1) and (S2).
Fallback steps satisfy them too: recovery steps by definition, core paths by (K1) and (K2), and the path to a goal by (E3).
Lemma~\ref{lem:safe} then gives safety, and Lemma~\ref{lem:invariant} gives the recovery plan.

Next, we prove progress guarantee.
Let $g$ be a goal that is never withdrawn or replaced.
New goals receive larger tickets, so while $g$ is active, the oldest active goal is $g$ or one of the at most $n-1$ goals that were active when $g$ was accepted.
Each such goal ends within $D$ seconds of becoming the oldest active goal, or CGAR pauses normal routing for it at the first routing decision at or after its fallback deadline and it ends by Lemma~\ref{lem:fallback}.
At most $n-1$ older goals can precede $g$, and each ends in finite time.
Thus $g$ is completed in finite time.
\end{proof}

\clearpage
\section{Evaluation on Additional Layouts}
\label{app:heldout}

\subsection{PLATM on RANDOM-06 to RANDOM-10}
\label{app:heldout:lorr}

We generate one new $32\times32$ layout with the size and obstacle density of the competition map, and reuse the five fleet sizes and horizons of Section~\ref{sec:lorr:problem_setup} to obtain RANDOM-06 to RANDOM-10 (Figure~\ref{fig:heldout:random}).

\par\medskip\noindent
\begin{minipage}{\linewidth}
\centering
\begin{minipage}{0.195\linewidth}\centering
\resizebox{\linewidth}{!}{\begin{tikzpicture}[x=0.071cm,y=0.071cm]
\path[fill=white,draw=black!40,line width=0.2pt] (0,0) rectangle (32,32);
\path[fill=black!62] (1,31) rectangle (2,32) (9,31) rectangle (10,32) (28,31) rectangle (29,32) (30,31) rectangle (31,32) (1,30) rectangle (2,31) (4,30) rectangle (5,31) (7,30) rectangle (8,31) (12,30) rectangle (13,31) (14,30) rectangle (15,31) (16,30) rectangle (17,31) (21,30) rectangle (22,31) (30,30) rectangle (31,31) (8,29) rectangle (9,30) (12,29) rectangle (13,30) (15,29) rectangle (16,30) (16,29) rectangle (17,30) (20,29) rectangle (21,30) (24,29) rectangle (25,30) (27,29) rectangle (28,30) (28,29) rectangle (29,30) (29,29) rectangle (30,30) (0,28) rectangle (1,29) (1,28) rectangle (2,29) (5,28) rectangle (6,29) (9,28) rectangle (10,29) (12,28) rectangle (13,29) (15,28) rectangle (16,29) (17,28) rectangle (18,29) (20,28) rectangle (21,29) (22,28) rectangle (23,29) (3,27) rectangle (4,28) (4,27) rectangle (5,28) (13,27) rectangle (14,28) (24,27) rectangle (25,28) (25,27) rectangle (26,28) (2,26) rectangle (3,27) (5,26) rectangle (6,27) (6,26) rectangle (7,27) (12,26) rectangle (13,27) (22,26) rectangle (23,27) (27,26) rectangle (28,27) (3,25) rectangle (4,26) (5,25) rectangle (6,26) (8,25) rectangle (9,26) (12,25) rectangle (13,26) (13,25) rectangle (14,26) (22,25) rectangle (23,26) (24,25) rectangle (25,26) (30,25) rectangle (31,26) (22,24) rectangle (23,25) (28,24) rectangle (29,25) (4,23) rectangle (5,24) (13,23) rectangle (14,24) (14,23) rectangle (15,24) (16,23) rectangle (17,24) (29,23) rectangle (30,24) (1,22) rectangle (2,23) (3,22) rectangle (4,23) (11,22) rectangle (12,23) (13,22) rectangle (14,23) (23,22) rectangle (24,23) (31,22) rectangle (32,23) (0,21) rectangle (1,22) (4,21) rectangle (5,22) (7,21) rectangle (8,22) (8,21) rectangle (9,22) (11,21) rectangle (12,22) (13,21) rectangle (14,22) (14,21) rectangle (15,22) (19,21) rectangle (20,22) (22,21) rectangle (23,22) (29,21) rectangle (30,22) (2,20) rectangle (3,21) (4,20) rectangle (5,21) (7,20) rectangle (8,21) (12,20) rectangle (13,21) (19,20) rectangle (20,21) (4,19) rectangle (5,20) (6,19) rectangle (7,20) (7,19) rectangle (8,20) (20,19) rectangle (21,20) (24,19) rectangle (25,20) (26,19) rectangle (27,20) (28,19) rectangle (29,20) (0,18) rectangle (1,19) (2,18) rectangle (3,19) (7,18) rectangle (8,19) (8,18) rectangle (9,19) (14,18) rectangle (15,19) (18,18) rectangle (19,19) (20,18) rectangle (21,19) (14,17) rectangle (15,18) (19,17) rectangle (20,18) (22,17) rectangle (23,18) (31,17) rectangle (32,18) (0,16) rectangle (1,17) (2,16) rectangle (3,17) (11,16) rectangle (12,17) (18,16) rectangle (19,17) (22,16) rectangle (23,17) (23,16) rectangle (24,17) (24,16) rectangle (25,17) (25,16) rectangle (26,17) (26,16) rectangle (27,17) (28,16) rectangle (29,17) (2,15) rectangle (3,16) (9,15) rectangle (10,16) (18,15) rectangle (19,16) (19,15) rectangle (20,16) (8,14) rectangle (9,15) (11,14) rectangle (12,15) (12,14) rectangle (13,15) (19,14) rectangle (20,15) (29,14) rectangle (30,15) (31,14) rectangle (32,15) (2,13) rectangle (3,14) (14,13) rectangle (15,14) (16,13) rectangle (17,14) (19,13) rectangle (20,14) (20,13) rectangle (21,14) (21,13) rectangle (22,14) (22,13) rectangle (23,14) (27,13) rectangle (28,14) (28,13) rectangle (29,14) (5,12) rectangle (6,13) (8,12) rectangle (9,13) (11,12) rectangle (12,13) (12,12) rectangle (13,13) (18,12) rectangle (19,13) (28,12) rectangle (29,13) (29,12) rectangle (30,13) (9,11) rectangle (10,12) (31,11) rectangle (32,12) (5,10) rectangle (6,11) (14,10) rectangle (15,11) (25,10) rectangle (26,11) (1,9) rectangle (2,10) (2,9) rectangle (3,10) (3,9) rectangle (4,10) (11,9) rectangle (12,10) (13,9) rectangle (14,10) (15,9) rectangle (16,10) (24,9) rectangle (25,10) (27,9) rectangle (28,10) (30,9) rectangle (31,10) (3,8) rectangle (4,9) (15,8) rectangle (16,9) (23,8) rectangle (24,9) (25,8) rectangle (26,9) (7,7) rectangle (8,8) (10,7) rectangle (11,8) (17,7) rectangle (18,8) (21,7) rectangle (22,8) (23,7) rectangle (24,8) (25,7) rectangle (26,8) (26,7) rectangle (27,8) (3,6) rectangle (4,7) (8,6) rectangle (9,7) (10,6) rectangle (11,7) (11,6) rectangle (12,7) (21,6) rectangle (22,7) (23,6) rectangle (24,7) (31,6) rectangle (32,7) (1,5) rectangle (2,6) (5,5) rectangle (6,6) (11,5) rectangle (12,6) (12,5) rectangle (13,6) (14,5) rectangle (15,6) (18,5) rectangle (19,6) (27,5) rectangle (28,6) (28,5) rectangle (29,6) (29,5) rectangle (30,6) (11,4) rectangle (12,5) (19,4) rectangle (20,5) (21,4) rectangle (22,5) (22,4) rectangle (23,5) (24,4) rectangle (25,5) (28,4) rectangle (29,5) (30,4) rectangle (31,5) (13,3) rectangle (14,4) (23,3) rectangle (24,4) (25,3) rectangle (26,4) (27,3) rectangle (28,4) (5,2) rectangle (6,3) (9,2) rectangle (10,3) (10,2) rectangle (11,3) (11,2) rectangle (12,3) (15,2) rectangle (16,3) (21,2) rectangle (22,3) (26,2) rectangle (27,3) (28,2) rectangle (29,3) (0,1) rectangle (1,2) (4,1) rectangle (5,2) (7,1) rectangle (8,2) (10,1) rectangle (11,2) (11,1) rectangle (12,2) (14,1) rectangle (15,2) (17,1) rectangle (18,2) (18,1) rectangle (19,2) (9,0) rectangle (10,1) (16,0) rectangle (17,1) (23,0) rectangle (24,1) (24,0) rectangle (25,1) (26,0) rectangle (27,1) (29,0) rectangle (30,1);
\path[fill=lorrBlue] (1.08,12.08) rectangle (1.92,12.92) (19.08,0.08) rectangle (19.92,0.92) (12.08,23.08) rectangle (12.92,23.92) (3.08,26.08) rectangle (3.92,26.92) (24.08,15.08) rectangle (24.92,15.92) (31.08,20.08) rectangle (31.92,20.92) (2.08,17.08) rectangle (2.92,17.92) (11.08,28.08) rectangle (11.92,28.92) (11.08,29.08) rectangle (11.92,29.92) (27.08,4.08) rectangle (27.92,4.92) (19.08,1.08) rectangle (19.92,1.92) (25.08,30.08) rectangle (25.92,30.92) (21.08,11.08) rectangle (21.92,11.92) (7.08,31.08) rectangle (7.92,31.92) (30.08,18.08) rectangle (30.92,18.92) (13.08,19.08) rectangle (13.92,19.92) (27.08,31.08) rectangle (27.92,31.92) (31.08,21.08) rectangle (31.92,21.92) (23.08,21.08) rectangle (23.92,21.92) (5.08,19.08) rectangle (5.92,19.92) (4.08,22.08) rectangle (4.92,22.92) (2.08,7.08) rectangle (2.92,7.92) (20.08,14.08) rectangle (20.92,14.92) (21.08,1.08) rectangle (21.92,1.92) (4.08,28.08) rectangle (4.92,28.92) (14.08,28.08) rectangle (14.92,28.92) (12.08,3.08) rectangle (12.92,3.92) (15.08,0.08) rectangle (15.92,0.92) (5.08,11.08) rectangle (5.92,11.92) (5.08,24.08) rectangle (5.92,24.92) (2.08,4.08) rectangle (2.92,4.92) (15.08,14.08) rectangle (15.92,14.92) (16.08,16.08) rectangle (16.92,16.92) (10.08,28.08) rectangle (10.92,28.92) (28.08,30.08) rectangle (28.92,30.92) (26.08,25.08) rectangle (26.92,25.92) (16.08,12.08) rectangle (16.92,12.92) (16.08,19.08) rectangle (16.92,19.92) (8.08,15.08) rectangle (8.92,15.92) (2.08,28.08) rectangle (2.92,28.92) (0.08,9.08) rectangle (0.92,9.92) (4.08,3.08) rectangle (4.92,3.92) (1.08,18.08) rectangle (1.92,18.92) (10.08,23.08) rectangle (10.92,23.92) (26.08,26.08) rectangle (26.92,26.92) (4.08,2.08) rectangle (4.92,2.92) (31.08,30.08) rectangle (31.92,30.92) (31.08,5.08) rectangle (31.92,5.92) (9.08,16.08) rectangle (9.92,16.92) (22.08,10.08) rectangle (22.92,10.92) (12.08,4.08) rectangle (12.92,4.92) (13.08,15.08) rectangle (13.92,15.92) (24.08,1.08) rectangle (24.92,1.92) (8.08,4.08) rectangle (8.92,4.92) (28.08,28.08) rectangle (28.92,28.92) (5.08,7.08) rectangle (5.92,7.92) (8.08,22.08) rectangle (8.92,22.92) (0.08,20.08) rectangle (0.92,20.92) (7.08,3.08) rectangle (7.92,3.92) (26.08,31.08) rectangle (26.92,31.92) (30.08,24.08) rectangle (30.92,24.92) (3.08,13.08) rectangle (3.92,13.92) (7.08,15.08) rectangle (7.92,15.92) (5.08,22.08) rectangle (5.92,22.92) (13.08,31.08) rectangle (13.92,31.92) (28.08,22.08) rectangle (28.92,22.92) (9.08,7.08) rectangle (9.92,7.92) (6.08,16.08) rectangle (6.92,16.92) (17.08,14.08) rectangle (17.92,14.92) (27.08,2.08) rectangle (27.92,2.92) (29.08,25.08) rectangle (29.92,25.92) (2.08,5.08) rectangle (2.92,5.92) (5.08,17.08) rectangle (5.92,17.92) (16.08,21.08) rectangle (16.92,21.92) (0.08,14.08) rectangle (0.92,14.92) (18.08,24.08) rectangle (18.92,24.92) (29.08,22.08) rectangle (29.92,22.92) (2.08,10.08) rectangle (2.92,10.92) (3.08,23.08) rectangle (3.92,23.92) (11.08,15.08) rectangle (11.92,15.92) (8.08,19.08) rectangle (8.92,19.92) (18.08,14.08) rectangle (18.92,14.92) (20.08,4.08) rectangle (20.92,4.92) (3.08,31.08) rectangle (3.92,31.92) (5.08,16.08) rectangle (5.92,16.92) (30.08,21.08) rectangle (30.92,21.92) (1.08,24.08) rectangle (1.92,24.92) (31.08,1.08) rectangle (31.92,1.92) (24.08,3.08) rectangle (24.92,3.92) (15.08,22.08) rectangle (15.92,22.92) (17.08,24.08) rectangle (17.92,24.92) (29.08,7.08) rectangle (29.92,7.92) (12.08,31.08) rectangle (12.92,31.92) (30.08,17.08) rectangle (30.92,17.92) (8.08,31.08) rectangle (8.92,31.92) (29.08,10.08) rectangle (29.92,10.92) (24.08,13.08) rectangle (24.92,13.92) (16.08,15.08) rectangle (16.92,15.92) (10.08,26.08) rectangle (10.92,26.92) (18.08,31.08) rectangle (18.92,31.92);
\path[fill=white] (1.32,12.23) -- (1.78,12.50) -- (1.32,12.77) -- cycle (19.32,0.23) -- (19.78,0.50) -- (19.32,0.77) -- cycle (12.32,23.23) -- (12.78,23.50) -- (12.32,23.77) -- cycle (3.32,26.23) -- (3.78,26.50) -- (3.32,26.77) -- cycle (24.32,15.23) -- (24.78,15.50) -- (24.32,15.77) -- cycle (31.32,20.23) -- (31.78,20.50) -- (31.32,20.77) -- cycle (2.32,17.23) -- (2.78,17.50) -- (2.32,17.77) -- cycle (11.32,28.23) -- (11.78,28.50) -- (11.32,28.77) -- cycle (11.32,29.23) -- (11.78,29.50) -- (11.32,29.77) -- cycle (27.32,4.23) -- (27.78,4.50) -- (27.32,4.77) -- cycle (19.32,1.23) -- (19.78,1.50) -- (19.32,1.77) -- cycle (25.32,30.23) -- (25.78,30.50) -- (25.32,30.77) -- cycle (21.32,11.23) -- (21.78,11.50) -- (21.32,11.77) -- cycle (7.32,31.23) -- (7.78,31.50) -- (7.32,31.77) -- cycle (30.32,18.23) -- (30.78,18.50) -- (30.32,18.77) -- cycle (13.32,19.23) -- (13.78,19.50) -- (13.32,19.77) -- cycle (27.32,31.23) -- (27.78,31.50) -- (27.32,31.77) -- cycle (31.32,21.23) -- (31.78,21.50) -- (31.32,21.77) -- cycle (23.32,21.23) -- (23.78,21.50) -- (23.32,21.77) -- cycle (5.32,19.23) -- (5.78,19.50) -- (5.32,19.77) -- cycle (4.32,22.23) -- (4.78,22.50) -- (4.32,22.77) -- cycle (2.32,7.23) -- (2.78,7.50) -- (2.32,7.77) -- cycle (20.32,14.23) -- (20.78,14.50) -- (20.32,14.77) -- cycle (21.32,1.23) -- (21.78,1.50) -- (21.32,1.77) -- cycle (4.32,28.23) -- (4.78,28.50) -- (4.32,28.77) -- cycle (14.32,28.23) -- (14.78,28.50) -- (14.32,28.77) -- cycle (12.32,3.23) -- (12.78,3.50) -- (12.32,3.77) -- cycle (15.32,0.23) -- (15.78,0.50) -- (15.32,0.77) -- cycle (5.32,11.23) -- (5.78,11.50) -- (5.32,11.77) -- cycle (5.32,24.23) -- (5.78,24.50) -- (5.32,24.77) -- cycle (2.32,4.23) -- (2.78,4.50) -- (2.32,4.77) -- cycle (15.32,14.23) -- (15.78,14.50) -- (15.32,14.77) -- cycle (16.32,16.23) -- (16.78,16.50) -- (16.32,16.77) -- cycle (10.32,28.23) -- (10.78,28.50) -- (10.32,28.77) -- cycle (28.32,30.23) -- (28.78,30.50) -- (28.32,30.77) -- cycle (26.32,25.23) -- (26.78,25.50) -- (26.32,25.77) -- cycle (16.32,12.23) -- (16.78,12.50) -- (16.32,12.77) -- cycle (16.32,19.23) -- (16.78,19.50) -- (16.32,19.77) -- cycle (8.32,15.23) -- (8.78,15.50) -- (8.32,15.77) -- cycle (2.32,28.23) -- (2.78,28.50) -- (2.32,28.77) -- cycle (0.32,9.23) -- (0.78,9.50) -- (0.32,9.77) -- cycle (4.32,3.23) -- (4.78,3.50) -- (4.32,3.77) -- cycle (1.32,18.23) -- (1.78,18.50) -- (1.32,18.77) -- cycle (10.32,23.23) -- (10.78,23.50) -- (10.32,23.77) -- cycle (26.32,26.23) -- (26.78,26.50) -- (26.32,26.77) -- cycle (4.32,2.23) -- (4.78,2.50) -- (4.32,2.77) -- cycle (31.32,30.23) -- (31.78,30.50) -- (31.32,30.77) -- cycle (31.32,5.23) -- (31.78,5.50) -- (31.32,5.77) -- cycle (9.32,16.23) -- (9.78,16.50) -- (9.32,16.77) -- cycle (22.32,10.23) -- (22.78,10.50) -- (22.32,10.77) -- cycle (12.32,4.23) -- (12.78,4.50) -- (12.32,4.77) -- cycle (13.32,15.23) -- (13.78,15.50) -- (13.32,15.77) -- cycle (24.32,1.23) -- (24.78,1.50) -- (24.32,1.77) -- cycle (8.32,4.23) -- (8.78,4.50) -- (8.32,4.77) -- cycle (28.32,28.23) -- (28.78,28.50) -- (28.32,28.77) -- cycle (5.32,7.23) -- (5.78,7.50) -- (5.32,7.77) -- cycle (8.32,22.23) -- (8.78,22.50) -- (8.32,22.77) -- cycle (0.32,20.23) -- (0.78,20.50) -- (0.32,20.77) -- cycle (7.32,3.23) -- (7.78,3.50) -- (7.32,3.77) -- cycle (26.32,31.23) -- (26.78,31.50) -- (26.32,31.77) -- cycle (30.32,24.23) -- (30.78,24.50) -- (30.32,24.77) -- cycle (3.32,13.23) -- (3.78,13.50) -- (3.32,13.77) -- cycle (7.32,15.23) -- (7.78,15.50) -- (7.32,15.77) -- cycle (5.32,22.23) -- (5.78,22.50) -- (5.32,22.77) -- cycle (13.32,31.23) -- (13.78,31.50) -- (13.32,31.77) -- cycle (28.32,22.23) -- (28.78,22.50) -- (28.32,22.77) -- cycle (9.32,7.23) -- (9.78,7.50) -- (9.32,7.77) -- cycle (6.32,16.23) -- (6.78,16.50) -- (6.32,16.77) -- cycle (17.32,14.23) -- (17.78,14.50) -- (17.32,14.77) -- cycle (27.32,2.23) -- (27.78,2.50) -- (27.32,2.77) -- cycle (29.32,25.23) -- (29.78,25.50) -- (29.32,25.77) -- cycle (2.32,5.23) -- (2.78,5.50) -- (2.32,5.77) -- cycle (5.32,17.23) -- (5.78,17.50) -- (5.32,17.77) -- cycle (16.32,21.23) -- (16.78,21.50) -- (16.32,21.77) -- cycle (0.32,14.23) -- (0.78,14.50) -- (0.32,14.77) -- cycle (18.32,24.23) -- (18.78,24.50) -- (18.32,24.77) -- cycle (29.32,22.23) -- (29.78,22.50) -- (29.32,22.77) -- cycle (2.32,10.23) -- (2.78,10.50) -- (2.32,10.77) -- cycle (3.32,23.23) -- (3.78,23.50) -- (3.32,23.77) -- cycle (11.32,15.23) -- (11.78,15.50) -- (11.32,15.77) -- cycle (8.32,19.23) -- (8.78,19.50) -- (8.32,19.77) -- cycle (18.32,14.23) -- (18.78,14.50) -- (18.32,14.77) -- cycle (20.32,4.23) -- (20.78,4.50) -- (20.32,4.77) -- cycle (3.32,31.23) -- (3.78,31.50) -- (3.32,31.77) -- cycle (5.32,16.23) -- (5.78,16.50) -- (5.32,16.77) -- cycle (30.32,21.23) -- (30.78,21.50) -- (30.32,21.77) -- cycle (1.32,24.23) -- (1.78,24.50) -- (1.32,24.77) -- cycle (31.32,1.23) -- (31.78,1.50) -- (31.32,1.77) -- cycle (24.32,3.23) -- (24.78,3.50) -- (24.32,3.77) -- cycle (15.32,22.23) -- (15.78,22.50) -- (15.32,22.77) -- cycle (17.32,24.23) -- (17.78,24.50) -- (17.32,24.77) -- cycle (29.32,7.23) -- (29.78,7.50) -- (29.32,7.77) -- cycle (12.32,31.23) -- (12.78,31.50) -- (12.32,31.77) -- cycle (30.32,17.23) -- (30.78,17.50) -- (30.32,17.77) -- cycle (8.32,31.23) -- (8.78,31.50) -- (8.32,31.77) -- cycle (29.32,10.23) -- (29.78,10.50) -- (29.32,10.77) -- cycle (24.32,13.23) -- (24.78,13.50) -- (24.32,13.77) -- cycle (16.32,15.23) -- (16.78,15.50) -- (16.32,15.77) -- cycle (10.32,26.23) -- (10.78,26.50) -- (10.32,26.77) -- cycle (18.32,31.23) -- (18.78,31.50) -- (18.32,31.77) -- cycle;
\end{tikzpicture}
}\par\vspace{1pt}
\maplabel{\textbf{RANDOM-06}\\100 robots}
\end{minipage}\hfill
\begin{minipage}{0.195\linewidth}\centering
\resizebox{\linewidth}{!}{\input{assets/heldout/random-07-map.tex}}\par\vspace{1pt}
\maplabel{\textbf{RANDOM-07}\\200 robots}
\end{minipage}\hfill
\begin{minipage}{0.195\linewidth}\centering
\resizebox{\linewidth}{!}{\input{assets/heldout/random-08-map.tex}}\par\vspace{1pt}
\maplabel{\textbf{RANDOM-08}\\400 robots}
\end{minipage}\hfill
\begin{minipage}{0.195\linewidth}\centering
\resizebox{\linewidth}{!}{\input{assets/heldout/random-09-map.tex}}\par\vspace{1pt}
\maplabel{\textbf{RANDOM-09}\\700 robots}
\end{minipage}\hfill
\begin{minipage}{0.195\linewidth}\centering
\resizebox{\linewidth}{!}{\input{assets/heldout/random-10-map.tex}}\par\vspace{1pt}
\maplabel{\textbf{RANDOM-10}\\800 robots}
\end{minipage}
\captionof{figure}{The five instances RANDOM-06 to RANDOM-10 at their start.
Gray cells are obstacles, and blue cells hold robots.}
\label{fig:heldout:random}
\end{minipage}
\par\medskip

Table~\ref{tab:heldout:lorr:adaptation} reports PLATM on both layouts over eight paired trials with fresh task streams.
The new maps received map-specific calibration, as it is part of the algorithm.
Their task counts are within 3.9\% of those on the competition layout.

\newcommand{\LorrAdaptationRows}{%
RANDOM-01 & 100 & 736.8 & 10.4 & 724 & 754 \\
RANDOM-02 & 200 & 1,356.9 & 14.6 & 1,332 & 1,380 \\
RANDOM-03 & 400 & 2,647.5 & 31.4 & 2,606 & 2,703 \\
RANDOM-04 & 700 & 2,790.1 & 54.9 & 2,716 & 2,886 \\
RANDOM-05 & 800 & 4,253.3 & 59.8 & 4,170 & 4,331 \\
\midrule
RANDOM-06 & 100 & 746.4 & 10.9 & 724 & 761 \\
RANDOM-07 & 200 & 1,373.0 & 16.7 & 1,345 & 1,400 \\
RANDOM-08 & 400 & 2,561.9 & 50.4 & 2,496 & 2,625 \\
RANDOM-09 & 700 & 2,707.6 & 31.6 & 2,658 & 2,749 \\
RANDOM-10 & 800 & 4,089.4 & 63.9 & 4,008 & 4,205 \\%
}

\begin{table}[H]
\centering
\caption{Tasks completed by PLATM on the competition layout and, after map-specific calibration, on the additional layout, as mean $\pm$ s.d.\ over eight paired trials.
The \mbox{RANDOM-09} and \mbox{RANDOM-10} profiles received extended calibration.}
\label{tab:heldout:lorr:adaptation}
\small
\setlength{\tabcolsep}{4pt}
\begin{tabular}{@{}l r r@{\,\(\pm\)\,}l rr@{}}
\toprule
 & \multicolumn{1}{c}{$n$} & Mean & S.D. & Min & Max \\
\midrule
\LorrAdaptationRows
\bottomrule
\end{tabular}
\end{table}

\subsection{CGAR on WAREHOUSE-00}
\label{app:heldout:warehouse}

We run CGAR unchanged on WAREHOUSE-00, the synthetic layout of Figure~\ref{fig:heldout:warehouse}, with 16 robots and the demand seeds of the main comparison.
Table~\ref{tab:heldout:warehouse} reports throughput per robot: WAREHOUSE-00 reaches $16.98\pm0.78$ orders per hour per robot, on par with the real warehouses of similar fleet size.
All eight runs finish with no logged collisions or assignment failures.

\par\medskip\noindent
\begin{minipage}{\linewidth}
\centering
\includegraphics[width=0.55\linewidth]{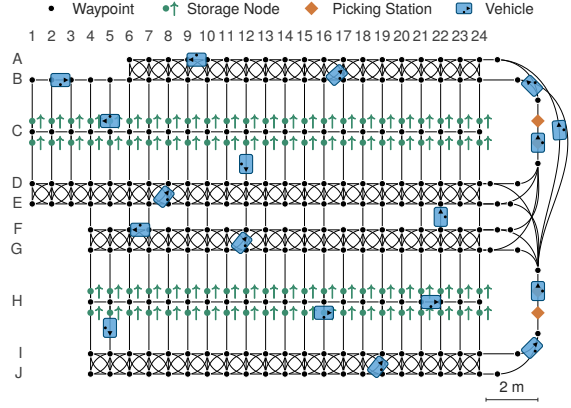}
\captionof{figure}{WAREHOUSE-00, the synthetic layout evaluated here.
Black lanes connect waypoints, green markers denote storage nodes, and orange diamonds denote picking stations.
Blue vehicles illustrate poses and are not the evaluation's initial state.}
\label{fig:heldout:warehouse}
\end{minipage}
\par\medskip

\begin{table}[H]
\centering
\caption{Orders completed per hour per robot by CGAR, as mean $\pm$ s.d.\ over eight trials.
The symbols $n$ and $d$ are those of Table~\ref{tab:cgar:final}.}
\label{tab:heldout:warehouse}
\small
\setlength{\tabcolsep}{4pt}
\begin{tabular}{@{}l rr r@{\,\(\pm\)\,}l rr@{}}
\toprule
 & \multicolumn{1}{c}{$n$} & \multicolumn{1}{c}{$d$} & Mean & S.D. & Min & Max \\
\midrule
WAREHOUSE-01 & 65 & 51.8 & 11.93 & 0.39 & 11.28 & 12.44 \\
WAREHOUSE-02 & 17 & 43.5 & 16.39 & 0.16 & 16.18 & 16.59 \\
WAREHOUSE-03 &  8 & 21.4 & 16.89 & 0.72 & 16.06 & 18.00 \\
WAREHOUSE-04 & 19 & 49.3 & 14.18 & 0.46 & 13.39 & 15.00 \\
WAREHOUSE-05 & 13 & 38.3 & 14.73 & 0.57 & 13.65 & 15.50 \\
WAREHOUSE-06 & 32 & 74.2 &  9.52 & 0.39 &  9.06 & 10.02 \\
\midrule
WAREHOUSE-00 & 16 & 32.2 & 16.98 & 0.78 & 15.88 & 17.97 \\
\bottomrule
\end{tabular}
\end{table}

\end{document}